\documentclass{article}

\usepackage[nonatbib,final]{nips_2017}

\usepackage[utf8]{inputenc} 
\usepackage[T1]{fontenc}    
\usepackage{hyperref}       
\usepackage{url}            
\usepackage{booktabs}       
\usepackage{amsfonts}       
\usepackage{nicefrac}       
\usepackage{microtype}      

\usepackage[ruled]{algorithm2e}
\usepackage{graphicx}
\usepackage{amsmath}
\usepackage{amssymb}
\usepackage{array}
\usepackage{subfigure}
\usepackage{amsthm}
\usepackage{color}
\usepackage{mathtools}
\usepackage{tikz}
\usepackage[utf8]{inputenc}
\usepackage{pgfplots}

\newcommand{\norm}[1]{\|{#1}\|}
\newcommand{\rank}{\operatorname{rank}}
\newcommand{\vect}{\operatorname{vec}}
\newcommand{\var}{\operatorname{Var}}
\newcommand{\diag}{\operatorname{diag}}
\newcommand{\ind}[1]{\mathbf{1}\left(#1\right)}
\newcommand{\calB}{\mathcal{B}}
\newcommand{\rmd}{\mathrm{d}}
\newcommand{\rmF}{\mathrm{F}}
\newcommand{\bbR}{\mathbb{R}}
\newcommand{\bbE}{\mathbb{E}}
\newcommand{\bbP}{\mathbb{P}}
\newcommand{\tU}{\widetilde{U}}
\newcommand{\tV}{\widetilde{V}}
\newcommand{\ba}{\bar{a}}
\newcommand{\bb}{\bar{b}}
\newcommand{\ta}{\tilde{a}}
\newcommand{\tb}{\tilde{b}}
\newcommand{\hX}{\widehat{X}}
\newcommand{\hU}{\widehat{U}}
\newcommand{\hV}{\widehat{V}}
\newcommand{\hSigma}{\widehat{\Sigma}}

\newtheorem{theorem}{Theorem}
\newtheorem{lemma}{Lemma}
\newtheorem{corollary}{Corollary}
\newtheorem{proposition}{Proposition}
\newtheorem{definition}{Definition}

\title{Joint Dimensionality Reduction\\for Separable Embedding Estimation}

\author{
  Yanjun~Li\\
  CSL and Department of ECE\\
  University of Illinois, Urbana-Champaign\\
  Urbana, IL 61801 \\
  \texttt{yli145@illinois.edu} \\
	\And
  Bihan~Wen\\
  School of Electrical and Electronic Engineering\\
  Nanyang Technological University\\
  Singapore\\
  \texttt{bihan.wen@ntu.edu.sg} \\
  	\And
  Hao~Cheng\\
  School of Electrical and Electronic Engineering\\
  Nanyang Technological University\\
  Singapore\\
  \texttt{hao006@e.ntu.edu.sg} \\
	\And
  Yoram~Bresler\\
  CSL and Department of ECE\\
  University of Illinois, Urbana-Champaign\\
  Urbana, IL 61801 \\
  \texttt{ybresler@illinois.edu} \\
}

\begin{document}

\maketitle

\begin{abstract}

Low-dimensional embeddings for data from disparate sources play critical roles in multi-modal machine learning, multimedia information retrieval, and bioinformatics. In this paper, we propose a supervised dimensionality reduction method that learns linear embeddings jointly for two feature vectors representing data of different modalities or data from distinct types of entities. We also propose an efficient feature selection method that complements, and can be applied prior to, our joint dimensionality reduction method. Assuming that there exist true linear embeddings for these features, our analysis of the error in the learned linear embeddings provides theoretical guarantees that the dimensionality reduction method accurately estimates the true embeddings when certain technical conditions are satisfied and the number of samples is sufficiently large. The derived sample complexity results are echoed by numerical experiments. We apply the proposed dimensionality reduction method to gene-disease association, and predict unknown associations using kernel regression on the dimension-reduced feature vectors. Our approach compares favorably against other dimensionality reduction methods, and against a state-of-the-art method of bilinear regression for predicting gene-disease associations.
\end{abstract}

\section{Introduction}

Dimensionality reduction (DR) is the process of extracting features from high dimensional data, which is essential in big data applications. Unsupervised learning techniques aim to embed high-dimensional data into low-dimensional features that most accurately represent the original data. The literature on this topic is vast, from classical methods, such as principal component analysis (PCA) and multidimensional scaling (MDS), to more recent approaches, such as Isomap and locally-linear embedding \cite{Tenenbaum2000,Roweis2000}. On the other hand, supervised learning techniques -- a long line of work including linear discriminant analysis (LDA) and canonical correlation analysis (CCA) -- extract features that are most relevant to the dependent variables. This is also the goal of the work in this paper. 

More recently, sufficient DR methods \cite{Adragni2009} were proposed to identify a linear embedding $U\in\bbR^{n \times r}$ ($r < n$) of feature $x\in\bbR^n$, such that the reduced feature $U^Tx\in\bbR^r$ captures all the information in the conditional distribution $y|x$ or the conditional mean $\bbE[y|x]$ of some response $y$. Such methods include ordinary least squares (OLS) \cite{Li1989}, sliced inverse regression (SIR) \cite{Li1991}, principal Hessian direction (pHd) \cite{Li1992}, sliced average variance estimation (SAVE) \cite{Cook1991,Cook2000}, structure-adaptive approach \cite{Hristache2001}, parametric inverse regression (PIR) \cite{Bura2001}, iterative Hessian transformation \cite{Cook2002}, minimum average variance estimation (MAVE) \cite{Xia2002}, minimum discrepancy approach \cite{Cook2005}, intraslice covariance estimation \cite{Cook2006}, density based MAVE \cite{Xia2007}, contour projection approaches \cite{Luo2009}, and marginal high moment regression \cite{Yin2004}. We refer the reader to a review \cite{Ma2012} of supervised DR methods.

In multi-modal machine learning, multimedia information retrieval, and bioinformatics, there usually exist multiple sets of features with distinct characteristics. In some cases, features come from data of different modalities. For example, image retrieval models are usually trained on datasets with both images and associated tags (text) \cite{Chua2009}. In other applications, data may represent completely different entities. For example, both gene features and disease features are critical in gene-disease association \cite{Natarajan2014}. Joint dimensionality reduction of two sets of features naturally arises in these scenarios. 

One might be tempted to learn a low-dimensional embedding for a long feature vector $[a^T, b^T]^T$ formed by concatenating the two types of features $a$ and $b$. The single embedding will inevitably mix the features from different sources. However, separable embeddings for different types of data are more interpretable, and are preferable for some data mining tasks \cite{Chang2015}. In addition, for some applications, separability in DR is an effective regularization, which results in faster inference and better generalization ability. An alternative approach is to learn two linear embeddings for the two types of features separately \cite{Li2009}. Unfortunately, as shown later in this paper, applying off-the-shelf algorithms (e.g., pHd) to two types of features separately may fail to recover the correct embeddings, therefore cannot extract useful features.

For the aforementioned reasons, there is much interest in learning separable embeddings jointly for two feature vectors. Li et al. \cite{Li2010} extended MAVE and proposed to learn a block structured linear embedding. Like MAVE, the time complexity scales quadratically with the number of samples, which limits its applicability to large datasets. Moreover, their approach combines the linear embedding estimator with a non-parametric link function estimator, which heavily relies on the smoothness of the underlying link function. Naik and Tsai \cite{Naik2005} proposed to estimate SIR embeddings subject to linear constraints induced by groups of features. Guo et al. \cite{Guo2015} estimated a separable embedding that encloses the conventional unstructured embedding. Liu et al. \cite{Liu2016} applied the same idea to OLS, and derived a more efficient algorithm (sOLS) for learning separable embeddings. These methods all treat disparate feature types homogeneously by concatenating them into one long feature vector, and enforce a separable structure in the embedding. Other than groupwise MAVE \cite{Li2010}, the previous works on learning separable embedding did not give error analysis or sample complexity results.

In this paper, we propose a supervised DR method that jointly learns linear embeddings for disparate types of features. Since the embeddings correspond to different factors of a matrix (resp. a higher order tensor) formed by a weighted sum of outer products (resp. tensor products) of disparate types of features, separability arises naturally.
Our method can serve several practical purposes. First, the learned linear embeddings extract features that best explain the response, which is of interest to many data mining problems. Secondly, by learning a low-dimensional feature representation,
our method improves the computational efficiency and the generalization ability of learning algorithms, such as kernel regression and k-nearest neighbors classification. Lastly, the estimated embeddings can be used to initialize more sophisticated models (e.g., groupwise dimension reduction \cite{Li2010}, or heterogeneous network embedding \cite{Chang2015}). Our contributions are summarized as follows:

\begin{enumerate}
	\item We establish a novel joint dimensionality reduction (JDR) algorithm based on matrix factorization, where separability in feature embeddings arises naturally. The time complexity of our method scales linearly with the number of samples. Our method, which learns two embeddings jointly and can be extended to the case of multiple embeddings, is a new member in the family of supervised DR algorithms. 
	
	\item We also propose a feature selection method that can be applied prior to DR, and improves interpretability of the extracted features. As far as we know, this is the first work that combines separability and sparsity in feature embeddings.
	
	\item The error of the estimated linear embeddings is analyzed under some technical conditions, which leads to performance guarantees for our dimensionality reduction and feature selection algorithms in terms of sample complexities. Our analysis is an extension of the work on generalized linear models by Plan et al. \cite{Plan2016}. The theoretical results are complemented by numerical experiments on synthetic data. 
	
	\item We combine our JDR method with kernel regression, and apply this approach to predicting gene-disease associations. Our approach outperforms other DR approaches, as well as a state-of-the-art bilinear regression method for predicting gene-disease associations \cite{Natarajan2014}.
\end{enumerate}


\section{Embedding Models for Two Feature Vectors} \label{sec:model}

In this paper, we assume that the joint distribution of response $y$ and features $a\in\bbR^{n_1}$ and $b\in\bbR^{n_2}$ satisfies a low-dimensional embedding model. In particular, we assume that there exist linear embeddings $U\in\bbR^{n_1\times r}$ and $V\in\bbR^{n_2\times r}$ ($n_1>r$, $n_2>r$), such that $y$ depends on $a$ and $b$ only through low-dimensional features $U^T a\in \bbR^r$ and $V^T b \in\bbR^r$, i.e. we have the following Markov chain:
\begin{align}
(a,b) \rightarrow (U^T a, V^T b) \rightarrow y. \label{eq:markov}
\end{align}
In other words, the conditional distributions $y|(a,b)$ and $y|(U^Ta,V^Tb)$ are identical. 
Given \eqref{eq:markov}, there exists a deterministic bivariate functional $f(\cdot,\cdot)$ such that $\bbE[y|a,b] = \mu_{y|(a,b)} = f(U^T a, V^T b)$, and the randomness of $\mu_{y|(a,b)}$ comes from $U^T a$ and $V^T b$. The link function $f(\cdot,\cdot)$ can take many different forms, including a bilinear function $a^T UV^T b$ \cite{Natarajan2014}, a bilinear logistic function $1/(1+\exp(- a^T UV^T b - t))$ \cite{Chang2015}, or a radial basis function (RBF) kernel $\exp(-\norm{U^Ta-V^Tb}_2^2 / \sigma_\mathrm{RBF}^2)$.

A special case of the above model is very common in learning problems with two feature vectors. The more restrictive model assumes that $y$ depends on $a$ and $b$ only through $\mu_{y|(a,b)}$, i.e.,
\begin{align}
(a,b) \rightarrow (U^T a, V^T b) \rightarrow \mu_{y|(a,b)} = f(U^T a, V^T b) \rightarrow y, \label{eq:markov2}
\end{align}
Common examples of the conditional distribution $y|\mu_{y|(a,b)}$ include a Gaussian distribution $N(\mu_{y|(a,b)}, \sigma_c^2)$ (regression) and a Bernoulli distribution $\mathrm{Ber}(\mu_{y|(a,b)})$ (binary classification). 
In the rest of the paper, we assume only \eqref{eq:markov} in our analysis. The sole purpose of the special case \eqref{eq:markov2} is to demonstrate the connections of our model with popular learning models.

The goal of this paper is to introduce a supervised JDR method that does not assume any specific form of link function or conditional distribution. Other than the low-dimensional embedding assumption in \eqref{eq:markov}, the proposed JDR algorithm does not introduce any additional bias. When training data is abundant, our algorithm, combined with non-parametric learning methods, can potentially outperform conventional parametric models. For example, JDR combined with kernel regression leads to better prediction of gene-disease association than bilinear regression (see Section \ref{sec:gda}).


\section{Joint Dimensionality Reduction via Matrix Factorization}

In this section, we introduce a JDR algorithm (Algorithm \ref{alg:jdr}) that learns $r$-dimensional embeddings $U$ and $V$ from $m$ samples of both the features $\{a_i\}_{i=1}^m$, $\{b_i\}_{i=1}^m$, and the corresponding responses $\{y_i\}_{i=1}^m$. The parameter $r$ depends on data, and can be selected based on cross validation.

\begin{algorithm} \label{alg:jdr}
\DontPrintSemicolon
\SetKwInput{KwPara}{Parameter}
\caption{Joint Dimensionality Reduction (JDR)}
\KwIn{features $\{a_i\}_{i=1}^m\subset\bbR^{n_1}$, $\{b_i\}_{i=1}^m\subset\bbR^{n_2}$, responses $\{y_i\}_{i=1}^m\subset\bbR$}
\KwPara{$r$}
\textbf{1. Data normalization:} \;
1.1. Sample means: $\mu_a = \frac{1}{m} \sum_{i=1}^m a_i$,~~~$\mu_b = \frac{1}{m} \sum_{i=1}^m b_i$,~~~$\mu_y = \frac{1}{m} \sum_{i=1}^m y_i$ \;
1.2. Sample covariances: $\Sigma_a = \frac{1}{m} \sum_{i=1}^m (a_i-\mu_a)(a_i-\mu_a)^T$,~~~$\Sigma_b = \frac{1}{m} \sum_{i=1}^m (b_i-\mu_b)(b_i-\mu_b)^T$ \;
1.3. Cholesky decompositions: $C_aC_a^T = \mathrm{Cholesky}(\Sigma_a)$,~~~$C_bC_b^T = \mathrm{Cholesky}(\Sigma_b)$ \;
1.4. Normalization: $a_i' = C_a^{-1}(a_i-\mu_a)$,~~~$b_i' \leftarrow C_b^{-1}(b_i-\mu_b)$,~~~$y_i' \leftarrow y_i-\mu_y$ \;
\textbf{2. Embedding estimation:} \;
2.1. Embedding proxy: 
\vspace{-0.1in}
\begin{align}
\hX_0 \coloneqq \frac{1}{m} \sum_{i=1}^{m} a_i' y_i' b_i'^T \label{eq:lin}
\end{align}\vspace{-0.2in}\;
2.2. Compact singular value decomposition (SVD): $\hU'\hSigma' \hV'^T = \mathrm{SVD}_r(\hX_0)$ \;
2.3. Embeddings: $\hU = (C_a^T)^{-1}\hU'$,~~~$\hV=(C_b^T)^{-1}\hV'$ \;
\KwOut{$\hU \in \bbR^{n_1\times r}$, $\hV \in \bbR^{n_2\times r}$}
\end{algorithm}

\textbf{Error Analysis.} We derive error bounds for the learned embeddings $\hU$ and $\hV$ under the assumption that $\{a_i\}_{i=1}^m$ and $\{b_i\}_{i=1}^m$ are mutually independent sets of Gaussian random vectors. Although Gaussianity and independence are crucial to our theoretical analysis, Algorithm \ref{alg:jdr} produces good embeddings even if these assumptions are violated. (Experiments in the supplementary material confirm that our method can estimate the embeddings accurately even if (a) the distributions are non-Gaussian (e.g., uniform, Poisson); or (b) there is weak dependence between $\{a_i\}_{i=1}^m$ and $\{b_i\}_{i=1}^m$.)

The data normalization procedure creates uncorrelated feature vectors via an affine transform of the original data. If $\{a_i\}_{i=1}^m$ (resp. $\{b_i\}_{i=1}^m$) are i.i.d. Gaussian random vectors, and the true means and covariance matrices are used in data normalization, then $\{a_i'\}_{i=1}^m$ (resp. $\{b_i'\}_{i=1}^m$) are i.i.d. following standard Gaussian distribution $N(0,I_{n_1})$ (resp. $N(0,I_{n_2})$). 
Since the sample mean and sample covariance are consistent estimators for true mean and true covariance, the normalization error incurred by a finite sample size $m$ is inversely proportional to $\sqrt{m}$.

For simplicity, we only analyze the case where $\{a_i\}_{i=1}^m$ and $\{b_i\}_{i=1}^m$ follow standard Gaussian distributions, and data normalization in Algorithm \ref{alg:jdr} is deactivated, i.e., $a_i'=a_i$, $b_i'=b_i$, $y_i'=y_i$, and $\hU'=\hU$, $\hV'=\hV$ have orthonormal columns.
Theorem \ref{thm:lin} bounds the error between the true embeddings $U$, $V$ and the estimated embeddings $\hU$, $\hV$. Rather than recovering the exact embedding matrices $U$ and $V$, we care more about their range spaces, which we call dimensionality reduction subspaces. Without loss of generality, we assume that $U,V$ also have orthonormal columns. Therefore, we consider two embedding matrices $\hU$ and $U$ to be equivalent if their columns span the same subspace. Let $P_U$ denote the projection onto the column space of $U$. We use $d(U, \hU) \coloneqq \|\hU-P_{U}\hU\|_\rmF$ to quantify the subspace estimation error, which evaluates the residual of $\hU$ when projected onto the subspace encoded by $U$. Clearly, the estimation error is between $0$ and $\sqrt{r}$, attaining $0$ when $\hU$ and $U$ span the same subspace, and attaining $\sqrt{r}$ when the two subspaces are orthogonal.

\begin{theorem}\label{thm:lin}
Suppose features $\{a_i\}_{i=1}^m$ (resp. $\{b_i\}_{i=1}^m$) are i.i.d. random vectors following a Gaussian distribution $N(0,I_{n_1})$ (resp. $N(0,I_{n_2})$), and the responses $\{y_i\}_{i=1}^m$ are i.i.d. following the Markov chain model in \eqref{eq:markov}. Then the embedding proxy $\hX_0$ in \eqref{eq:lin} satisfies $\bbE \left[ \hX_0 \right] = UQV^T$ for some matrix $Q\in\bbR^{r\times r}$ determined by the link function $f(\cdot,\cdot)$. Moreover, if $y_i$, $y_ia_i$, $y_ib_i$, and $a_iy_ib_i^T$ have finite second moments, $Q$ is non-singular, and $r = O(1)$, then the estimated embeddings satisfy
\begin{align*}
\max\left\{\bbE\left[d(U,\hU)\right],~ \bbE\left[d(V,\hV)\right]\right\}= O\left(\sqrt{\frac{n_1n_2}{m}} \right).
\end{align*}
\end{theorem}

Theorem \ref{thm:lin} holds for any link function $f(\cdot,\cdot)$ such that $Q$ is non-singular. The dependence of the estimation error on dimension $r$ varies for different link functions, and is unknown since the link function is not specified. However, because $r=O(1)$, the dependence is hidden in constants. As an example, we derive the explicit dependence on $r$ for the bilinear link function $f(U^Ta,V^Tb)= a^TUV^Tb$, in which case $Q = I_r$. If the conditional variance $\var[y_i|a_i, b_i]\leq cr$ for some absolute constant $c$, then $\max\left\{\bbE\left[d(U,\hU)\right],~ \bbE\left[d(V,\hV)\right]\right\} \leq 2\sqrt{(c+1)r(n_1+2)(n_2+2)/m}$.

By Theorem \ref{thm:lin}, we need $m =O(n_1n_2)$ samples to produce an accurate estimate. However, when the responses $\{y_i\}_i^{m}$ are light-tailed random variables, we show that the sample complexity can be reduced to $m=O((n_1+n_2)\log^7(n_1+n_2))$.
\begin{definition}[Light-tailed Response Condition]\label{def:light}
We say $y_i$ satisfies the light-tailed response condition, if there exists absolute constants $c,C>0$, such that $\bbP\left[\left|y_i\right|\geq t\right] \leq C e^{-ct}$, for all $t\geq 0$.
\end{definition}

\begin{theorem}\label{thm:optimal}
Suppose the assumptions in Theorem \ref{thm:lin} are satisfied, and $\{y_i\}_{i=1}^m$ satisfy the light-tailed response condition, where $C>0$ and $c>\frac{1}{8\log(n_1+n_2)}$. If $m>n_1+n_2$, then
\begin{align*}
\max\left\{\bbE\left[d(U,\hU)\right],~ \bbE\left[d(V,\hV)\right] \right\} = O\left(\sqrt{\frac{(n_1+n_2)\log^2 m \log^4(n_1+n_2)}{m}}\right).
\end{align*}
\end{theorem}

Light-tailedness is a very mild condition. For example, the two models we examine in Section \ref{sec:synthetic} both satisfy this condition. Moreover, all random variables that are bounded almost surely are light-tailed (e.g., the association scores in Section \ref{sec:gda}). Light-tailedness can be checked empirically on data \cite{Bryson1974}.
When the condition is not satisfied, one can replace $y_i$ with a transformed response $y_i'=g(y_i)$ such that $y_i'$ is light-tailed. Owing to the Markov chain $(a,b) \rightarrow (U^T a, V^T b) \rightarrow y \rightarrow y'$, one can estimate $U$ and $V$ by applying Algorithm \ref{alg:jdr} to the transformed responses $\{y_i'\}_{i=1}^m$. 
The optimal choice of response transformation needs further investigation.

\textbf{Time Complexity and A Fast Approximate Algorithm.} The time complexities for data normalization and embedding estimation in Algorithm \ref{alg:jdr} are $O(mn_1^2 +  mn_2^2 + n_1^3 + n_2^3)$ and $O(mn_1n_2 + n_1n_2r + n_1^2r+n_2^2r)$, respectively. Assuming that $n_1 \approx n_2$ (on the order of $n$) and $m\geq n$, then the time complexities reduce to $O(mn^2)$. 

For very high dimensional data with weak correlation between features, one can use a feature-wise normalization similar to the batch normalization in deep learning \cite{Ioffe2015}, which computes only the diagonal entries of $\Sigma_a$ and $\Sigma_b$. The time complexity for this simplified data normalization is $O(mn)$.

At the same time, one can apply the randomized SVD algorithm \cite[Section 1.6]{Halko2011}, or a CountSketch based low-rank approximation algorithm \cite[Section 8]{Clarkson2013}, to the superposition of rank-$1$ terms in \eqref{eq:lin} without computing $\hX_0$. The resulting embedding estimation with randomized SVD (resp. CountSketch approximation) has time complexity $O(mnr)$ (resp. $O(mn)$).

Algorithm \ref{alg:jdr_fast} summarizes the fast JDR algorithm with feature-wise normalization and randomized SVD. The submatrix $\hU'^{(:,1:r)}$ (resp. $\hV'^{(:,1:r)}$) contains the first $r$ columns of $\hU'$ (resp. $\hV'$).
\begin{algorithm} \label{alg:jdr_fast}
\DontPrintSemicolon
\SetKwInput{KwPara}{Parameter}
\caption{Fast Joint Dimensionality Reduction}
\KwIn{features $\{a_i\}_{i=1}^m\subset\bbR^{n_1}$, $\{b_i\}_{i=1}^m\subset\bbR^{n_2}$, responses $\{y_i\}_{i=1}^m\subset\bbR$}
\KwPara{$r$}
\textbf{1. Feature-wise data normalization ($j\in[n_1]$, $k\in[n_2]$):} \;
1.1. Sample means: $\mu_a^{(j)} = \frac{1}{m} \sum_{i=1}^m a_i^{(j)}$,~~~$\mu_b^{(k)} = \frac{1}{m} \sum_{i=1}^m b_i^{(k)}$,~~~$\mu_y = \frac{1}{m} \sum_{i=1}^m y_i$ \;
1.2. Sample variances: $(\sigma_a^{(j)})^2 = \frac{1}{m} \sum_{i=1}^m (a_i^{(j)}-\mu_a^{(j)})^2$,~~~$(\sigma_b^{(k)})^2 = \frac{1}{m} \sum_{i=1}^m (b_i^{(k)}-\mu_b^{(k)})^2$ \;
1.3. Normalization: $a_i'^{(j)} = (a_i^{(j)}-\mu_a^{(j)})/\sigma_a^{(j)}$,~~~$b_i'^{(k)} \leftarrow (b_i^{(k)}-\mu_b^{(k)})/\sigma_b^{(k)}$,~~~$y_i' \leftarrow y_i-\mu_y$ \;
\textbf{2. Embedding estimation:} \;
2.1. Generate $n_2\times (2r)$ Gaussian random matrix $S$, and compute randomized proxy: 
\vspace{-0.1in}
\begin{align*}
Z = \frac{1}{m} \sum_{i=1}^{m} a_i' y_i' (b_i'^T S)
\end{align*}\vspace{-0.2in}\;
2.2. Economy size QR decomposition: $QR = \mathrm{QR}(Z)$ \;
2.3. Singular value decomposition: $\hU'\hSigma' \hV'^T = \mathrm{SVD}_{2r}(\sum_{i=1}^{m} (Q^T a_i') y_i' b_i'^T)$ \;
2.4. Embeddings: $\hU = \diag^{-1}[\sigma_a^{(1)},\dots,\sigma_a^{(n_1)}]Q\hU'^{(:,1:r)}$,~~~$\hV=\diag^{-1}[\sigma_b^{(1)},\dots,\sigma_b^{(n_2)}]\hV'^{(:,1:r)}$ \;
\KwOut{$\hU \in \bbR^{n_1\times r}$, $\hV \in \bbR^{n_2\times r}$}
\end{algorithm}

\textbf{Extension to Higher-Order Dimensionality Reduction and Multiple Responses.}
In applications where more than two types of features are present, one needs a higher-order joint dimensionality reduction method. Instead of the proxy matrix (tensor of order $2$) in \eqref{eq:lin}, one can compute a proxy tensor of higher order. For example, for three types of features $a_i\in\bbR^{n_1}$, $b_i\in\bbR^{n_2}$, and $c_i\in\bbR^{n_3}$, the corresponding $n_1\times n_2\times n_3$ tensor is $\frac{1}{m} \sum_{i=1}^{m} y_i a_i \otimes b_i \otimes c_i$, where $\otimes$ denotes the tensor product. To estimate the embedding matrices in higher-order DR, one can apply higher-order singular value decomposition (HOSVD) \cite{Bergqvist2010}. The embedding matrix of one type of features is the HOSVD factor in the mode corresponding to that feature type.

In some applications, the response $y_i$'s are vectors in stead of scalars (e.g., multiclass classification, regression with multiple dependent variables). In this case (assuming two feature vectors), the proxy tensor is computed as $\frac{1}{m} \sum_{i=1}^{m} y_i \otimes a_i \otimes b_i$. Similar to the previous case, the HOSVD factors in modes corresponding to features are the embedding matrices.


\section{Feature Selection}
Instead of extracting features from all variables in the original high dimensional data, one may consider selecting a smaller number of variables for learning. The combination of feature selection and extraction is equivalent to DR with sparsity constraints, which results in more interpretable features. Previous feature selection methods for supervised DR apply only to the case with one feature vector \cite{Li2007,Wang2008,Chen2010}. In this paper, we present a guaranteed greedy feature selection approach for two feature vectors. 

We assume that the embedding matrix $U$ (resp. $V$) in \eqref{eq:markov} has at most $s_1$ (resp. $s_2$) nonzero rows, where $r<s_1<n_1$ and $r<s_2<n_2$. Therefore, only $s_1$ features in $a$ and $s_2$ features in $b$ are active, and they are each reduced to $r$ features in $U^Ta$ and $V^Tb$, respectively. Let $\norm{\cdot}_0$ denote the number of nonzero entries in a vector or a matrix, and let $\norm{\cdot}_{0,r}$ and $\norm{\cdot}_{0,c}$ denote the numbers of nonzero rows and nonzero columns, respectively. We use $X^{(:,k)}$ to denote the $k$-th column of $X$ . The projection of matrix $Y$ onto set $\Omega$ is denoted by $P_{\Omega} Y \coloneqq \arg\min_{X\in\Omega} \norm{X-Y}_F$. Define
\begin{align}\label{eq:Omega}
\begin{split}
& \Omega_1 \coloneqq \{X\in\bbR^{n_1\times n_2}: \norm{X^{(:,k)}}_0\leq s_1,~\forall k\in [n_2]\}, \quad \Omega_2 \coloneqq \{X\in\bbR^{n_1\times n_2}: \norm{X}_{0,c}\leq s_2\}, \\
&\quad \Omega_3 \coloneqq \{X\in\bbR^{n_1\times n_2}: \norm{X}_{0,r}\leq s_1\}, \quad \Omega_r \coloneqq \{X\in\bbR^{n_1\times n_2}: \rank(X)\leq r\}.
\end{split}
\end{align}
We summarize the embedding estimation procedure with feature selection in Algorithm \ref{alg:gfs}.
\begin{algorithm} \label{alg:gfs}
\DontPrintSemicolon
\SetKwInput{KwPara}{Parameter}
\caption{Embedding Estimation with Feature Selection}
\KwIn{features $\{a_i\}_{i=1}^m\subset\bbR^{n_1}$, $\{b_i\}_{i=1}^m\subset\bbR^{n_2}$, responses $\{y_i\}_{i=1}^m\subset\bbR$}
\KwPara{$r$, $s_1$, $s_2$}
1. Compute embedding proxy $\hX_0 = \frac{1}{m} \sum_{i=1}^{m} a_i y_i b_i^T$ \;
2. Compute projections ($\Omega_1$, $\Omega_2$, and $\Omega_3$ are defined in \eqref{eq:Omega}, using $s_1$ and $s_2$): \vspace{-0.05in}
\[
\hX_1 \coloneqq P_{\Omega_1} \hX_0, \quad \hX_2 \coloneqq P_{\Omega_2} \hX_1, \quad \hX_3 \coloneqq P_{\Omega_3} \hX_2
\] \vspace{-0.2in}\;
3. Compute rank-$r$ approximation: $\hU \hSigma \hV^T = \mathrm{SVD}_r(\hX_3)$ \;
\KwOut{$\hU \in \bbR^{n_1\times r}$, $\hV \in \bbR^{n_2\times r}$}
\end{algorithm}

Step 2 in Algorithm \ref{alg:gfs} computes an $(s_1,s_2)$-sparse approximation $\hX_3$ for $\hX_0$, defined as an approximation with $s_1$ nonzero rows and $s_2$ nonzero columns. Finding the best $(s_1,s_2)$-sparse approximation of $\hX_0$ is in general NP-hard \cite{Khot2006}. Therefore, we use the approximate algorithm of sequential projection.
We compute the projection onto $\Omega_1$ by setting to zero all but the $s_1$ largest (in terms of absolute value) entries in each column of $\hX_0$; we compute the projection onto $\Omega_2$ by setting to zero all but the $s_2$ largest (in terms of $\ell_2$ norm) columns in $\hX_1$; we compute the projection onto $\Omega_3$ by setting to zero all but the $s_1$ largest (in terms of $\ell_2$ norm) rows in $\hX_2$.

The estimator satisfies that $\hU\hSigma\hV^T = P_{\Omega_r} P_{\Omega_3}P_{\Omega_2} P_{\Omega_1} \hX_0$, and that $\hU$ (resp. $\hV$) has $s_1$ (resp. $s_2$) nonzero rows. We bound the estimation error in Theorem \ref{thm:gfs}, which yields a sample complexity $m =O(s_1s_2\log n_1 \log n_2)$.

\begin{theorem}\label{thm:gfs}
Suppose the assumptions in Theorem \ref{thm:lin} are satisfied, and $U$ (resp. $V$) has at most $s_1$ (resp. $s_2$) nonzero rows. Then for $n_1,n_2 \geq 8$, the estimated embeddings in Algorithm \ref{alg:gfs} satisfy
\begin{align*}
\max\left\{\bbE\left[d(U,\hU)\right],~ \bbE\left[d(V,\hV)\right] \right\} = O\left(\sqrt{\frac{s_1s_2\log n_1 \log n_2}{m}}\right).
\end{align*}
\end{theorem}


\section{Application to Gene-Disease Association} \label{sec:application}

We apply the JDR method to predicting gene-disease associations \cite{Wu2008,Natarajan2014}. 
Suppose each disease $D_i$ ($i\in [m_D]$) has a feature vector $a_i\in\bbR^{n_D}$, and each gene $G_j$ ($j\in [m_G]$) has a feature vector $b_j\in\bbR^{n_G}$. 
Given an observed set $\Omega_\mathrm{obs} \subset [m_D]\times [m_G]$, with the gene-disease association score $y_{i,j}$ between $D_i$ and $G_j$ for each $(i,j) \in \Omega_\mathrm{obs}$, 
the goal is to predict the unknown associations in the unobserved set $\Omega_\mathrm{uno} = ([m_D]\times [m_G]) \backslash  \Omega_\mathrm{obs}$. This problem is an example of learning from dyadic data \cite{Hofmann1998}, which also arises in information retrieval, computational linguistics, and computer vision.

We propose to predict the unobserved gene-disease associations by combining JDR with kernel regression (JDR + KR). 
First, we apply Algorithm \ref{alg:jdr}, and compute the embedding proxy with normalized data, as $\hX_0 = \frac{1}{|\Omega_\mathrm{obs}|}\sum_{(i,j)\in \Omega_\mathrm{obs}} a_i'y_i'b_i'^T$. 
The embedded features for diseases and genes are computed as $a_i'' = \hSigma'^{1/2}\hU'^T a_i'$ and $b_i'' =\hSigma'^{1/2}\hV'^T b_i'$, respectively, where $\hU'$, $\hV'$, and $\hSigma'$ are defined in Step 2.2 in Algorithm \ref{alg:jdr}. 
Here we scale the $r$-dimensional embedded features with $\hSigma'^{1/2}$, therefore more prominent dimensionality reduction directions are given higher weights. 
Finally, we use Nadaraya-Watson kernel regression to predict unknown association scores. For $(i, j)\in \Omega_\mathrm{uno}$, the predicted score is
\vspace{-0.1in}
\[
\hat{y}_{i,j} = \frac{\sum_{(p,q) \in \Omega_\mathrm{obs}} K_{h_D}(a_i''-a_p'')K_{h_G}(b_j''-b_q'') y_{p,q}}{\sum_{(p,q) \in \Omega_\mathrm{obs}} K_{h_D}(a_i''-a_p'')K_{h_G}(b_j''-b_q'')},
\]
where $K_{h}(\cdot)$ denotes the Gaussian kernels in $\bbR^{r}$ with bandwidths $h$. 

Prior works applied network analysis to gene-disease prediction \cite{Barabasi2011, SinghBlom2013}, which have limited capability of making predictions for new diseases that do not belong to the networks. 
The recent method of inductive matrix completion (IMC) is based on bilinear regression, and demonstrated state-of-the-art gene-disease prediction performance \cite{Natarajan2014}. 
Compared to the parametric IMC model with explicit link function and loss function, the proposed prediction framework of dimensionality reduction plus non-parametric learning leads to better performance (see Section \ref{sec:gda}).


\section{Experiments} \label{sec:exp}

\subsection{Synthetic Data} \label{sec:synthetic}
In this section, we verify our theoretical analysis by examining the error in the estimated embedding matrices in Algorithms \ref{alg:jdr} and \ref{alg:gfs} with some numerical experiments. We use the normalized subspace estimation error (NSEE), defined by $\max\left\{d(U, \hU)/\sqrt{r},~ d(V,\hV)/\sqrt{r}\right\}$. We run experiments on two different models, both of which satisfy the light-tailed response condition: \\
\emph{(Bilinear model)} Let $\mu_i = f(U^Ta_i,V^Tb_i) = a_i^TUV^T b_i$, and $y_i = \mu_i+z_i$, where $\{z_i\}_{i=1}^m$ are i.i.d. Gaussian random variables $N(0,1)$.\\
\emph{(RBF model)} Let $\mu_i = f(U^Ta_i,V^Tb_i) = \exp(-\|U^T a_i-V^T b_i\|_2^2)$, and $y_i\sim\operatorname{Ber}(\mu_i)$ is a Bernoulli random variable with mean $\mu_i$.

In the experiments, we synthesize $a_i$ and $b_i$ using standard Gaussian random vectors, and the data normalization step is deactivated. 
We let $n_1=n_2\eqqcolon n$, $s_1=s_2 \eqqcolon s$, $r=5$, and synthesize true embeddings $U$ and $V$, each with $5$ orthonormal columns.  For each of the two models, we conduct four experiments. For Algorithm \ref{alg:jdr}, we fix the original feature dimension $n$ (resp. the sample size $m$), and study how error varies with $m$ (resp. $n$). For Algorithm \ref{alg:gfs} with feature selection, we fix $n$ and the sparsity level $s$ (resp. $n$ and $m$), and study how error varies with $m$ (resp. $s$). We repeat every experiment $100$ times, and show in Figure \ref{fig:log-log} the log-log plot (natural logarithm) of the mean NSEE versus $m$, $n$ or $s$. The results for the two models are roughly the same, which verifies that our algorithm and theory apply to different link functions or conditional distributions. The slopes of the plots in the first and third columns are roughly $-0.5$, which verifies the term $O(1/\sqrt{m})$ in the error bounds. The slopes of the plots in the second column are roughly $0.5$, which verifies the term $O(\sqrt{n_1+n_2})=O(\sqrt{n})$ in the error bound in Theorem \ref{thm:optimal}. The slopes of the plots in the fourth column are roughly $1$, which verifies the term $O(\sqrt{s_1s_2})=O(s)$ in the error bound in Theorem \ref{thm:gfs}. 

\begin{figure}[tb]
\centering
%
%
\begin{tikzpicture}[scale = 0.7]

\begin{axis}[%
width=1.268in,
height=1in,
at={(0in,0in)},
scale only axis,
xmin=6.90775527898214,
xmax=9.21034037197618,
xlabel={$\log(m)$},
ymin=-1.65133908669857,
ymax=0.164732059227385,
ylabel={$\log(\mathrm{NSEE})$},
axis background/.style={fill=white}
]
\addplot [color=blue,solid,line width=2.0pt,forget plot]
  table[row sep=crcr]{%
6.90775527898214	-0.215194441711561\\
7.60090245954208	-0.488714027390471\\
8.51719319141624	-0.929333134724102\\
9.21034037197618	-1.27141258575962\\
};
\end{axis}
\end{tikzpicture}%
~
%
%
\begin{tikzpicture}[scale = 0.7]

\begin{axis}[%
width=1.268in,
height=1in,
at={(0in,0in)},
scale only axis,
xmin=3.91202300542815,
xmax=6.21460809842219,
xlabel={$\log(n)$},
ymin=-1.96308556601509,
ymax=-0.147014420089143,
ylabel={$\log(\mathrm{NSEE})$},
axis background/.style={fill=white}
]
\addplot [color=blue,solid,line width=2.0pt,forget plot]
  table[row sep=crcr]{%
3.91202300542815	-1.63435455894625\\
4.60517018598809	-1.2740427772473\\
5.29831736654804	-0.923316299740511\\
6.21460809842219	-0.475745427157983\\
};
\end{axis}
\end{tikzpicture}%
~
%
%
\begin{tikzpicture}[scale = 0.7]

\begin{axis}[%
width=1.268in,
height=1in,
at={(0in,0in)},
scale only axis,
xmin=9.21034037197618,
xmax=11.5129254649702,
xlabel={$\log(m)$},
ymin=-4.08349742591548,
ymax=-2.26742627998953,
ylabel={$\log(\mathrm{NSEE})$},
axis background/.style={fill=white}
]
\addplot [color=blue,solid,line width=2.0pt,forget plot]
  table[row sep=crcr]{%
9.21034037197618	-2.55955806313597\\
9.90348755253613	-2.93485215962578\\
10.8197782844103	-3.39577582809748\\
11.5129254649702	-3.79136564276904\\
};
\end{axis}
\end{tikzpicture}%
~
%
%
\begin{tikzpicture}[scale = 0.7]

\begin{axis}[%
width=1.268in,
height=1in,
at={(0in,0in)},
scale only axis,
xmin=2.23694911729812,
xmax=4.67080616168402,
xlabel={$\log(s)$},
ymin=-3.33684670266371,
ymax=-1.41724009830129,
ylabel={$\log(\mathrm{NSEE})$},
axis background/.style={fill=white}
]
\addplot [color=blue,solid,line width=2.0pt,forget plot]
  table[row sep=crcr]{%
2.30258509299405	-3.33684670266371\\
2.99573227355399	-2.68991573176776\\
3.91202300542815	-2.0295046437283\\
4.60517018598809	-1.41724009830129\\
};
\end{axis}
\end{tikzpicture}%
\\
\vspace{-0.05in}
%
%
\begin{tikzpicture}[scale = 0.7]

\begin{axis}[%
width=1.268in,
height=1in,
at={(0in,0in)},
scale only axis,
xmin=9.90348755253613,
xmax=12.2060726455302,
xlabel={$\log(m)$},
ymin=-3.82992250251224,
ymax=-2.01385135658629,
ylabel={$\log(\mathrm{NSEE})$},
axis background/.style={fill=white}
]
\addplot [color=blue,solid,line width=2.0pt,forget plot]
  table[row sep=crcr]{%
9.90348755253613	-2.34853472496215\\
10.8197782844103	-2.8173177878998\\
11.5129254649702	-3.14841200116734\\
12.2060726455302	-3.49523913413638\\
};
\end{axis}
\end{tikzpicture}%
~
%
%
\begin{tikzpicture}[scale = 0.7]

\begin{axis}[%
width=1.268in,
height=1in,
at={(0in,0in)},
scale only axis,
xmin=2.30258509299405,
xmax=4.60517018598809,
xlabel={$\log(n)$},
ymin=-3.47200059299522,
ymax=-1.65592944706927,
ylabel={$\log(\mathrm{NSEE})$},
axis background/.style={fill=white}
]
\addplot [color=blue,solid,line width=2.0pt,forget plot]
  table[row sep=crcr]{%
2.30258509299405	-3.18594329230943\\
2.99573227355399	-2.78266797880402\\
3.91202300542815	-2.3097815954959\\
4.60517018598809	-1.94198674775507\\
};
\end{axis}
\end{tikzpicture}%
~
%
%
\begin{tikzpicture}[scale = 0.7]

\begin{axis}[%
width=1.268in,
height=1in,
at={(0in,0in)},
scale only axis,
xmin=12.2060726455302,
xmax=14.5086577385242,
xlabel={$\log(m)$},
ymin=-4.80549659137636,
ymax=-2.98942544545042,
ylabel={$\log(\mathrm{NSEE})$},
axis background/.style={fill=white}
]
\addplot [color=blue,solid,line width=2.0pt,forget plot]
  table[row sep=crcr]{%
12.2060726455302	-3.1982945629925\\
13.1223633774043	-3.79488492192837\\
13.8155105579643	-4.14964904055467\\
14.5086577385242	-4.59662747383429\\
};
\end{axis}
\end{tikzpicture}%
~
%
%
\begin{tikzpicture}[scale = 0.7]

\begin{axis}[%
width=1.268in,
height=1in,
at={(0in,0in)},
scale only axis,
xmin=1.49252226458343,
xmax=4.02893865327882,
xlabel={$\log(s)$},
ymin=-4.6840019801908,
ymax=-2.6835058284617,
ylabel={$\log(\mathrm{NSEE})$},
axis background/.style={fill=white}
]
\addplot [color=blue,solid,line width=2.0pt,forget plot]
  table[row sep=crcr]{%
1.6094379124341	-4.6840019801908\\
2.30258509299405	-4.17287904031326\\
2.99573227355399	-3.4981562412418\\
3.91202300542815	-2.6835058284617\\
};
\end{axis}
\end{tikzpicture}%
\vspace{-0.25in}
\caption{Log-log plots of mean NSEE versus $m$, $n$ or $s$. The two rows are for the bilinear model and the RBF model, respectively. Within each row, the four plots correspond to the four experiments.}
\label{fig:log-log}
\end{figure}
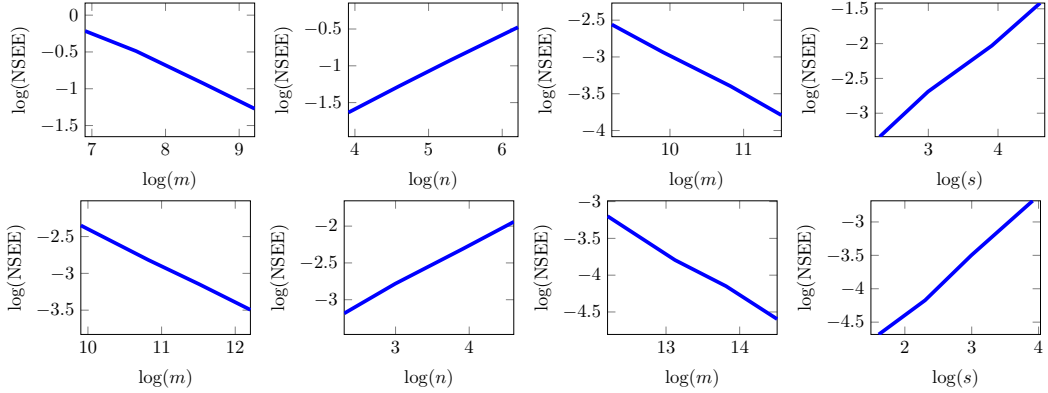

\begin{figure}[tb]
\centering
\subfigure{
\includegraphics[width=0.44\textwidth]{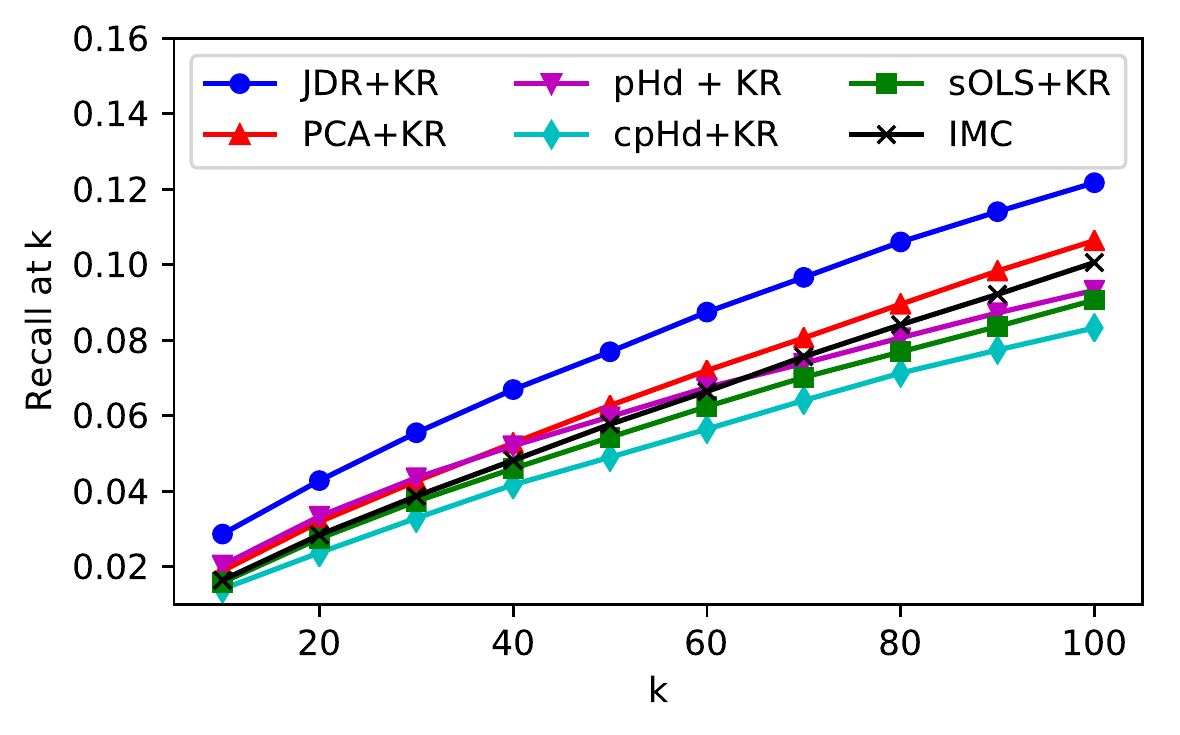}
}~~
\subfigure{
\includegraphics[width=0.44\textwidth]{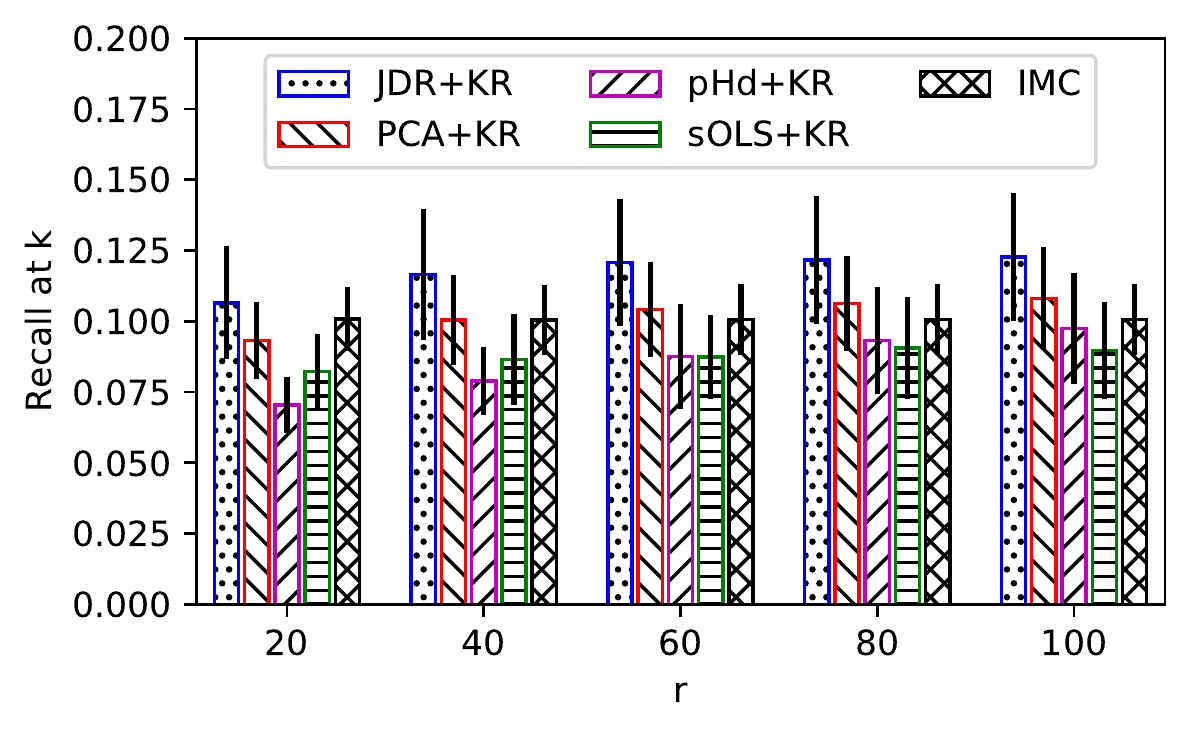}
}
\vspace{-0.15in}
\caption{Recall rates of different gene-disease association methods.}
\vspace{-0.15in}
\label{fig:linebar}
\end{figure}

\subsection{Gene-Disease Association} \label{sec:gda}

\textbf{Experiment Setup.}
In our experiments, $n_D=200$ features of $m_D=2358$ diseases are extracted via spectral embedding \cite{Ng2002} of the \emph{MimMiner} similarity matrix between diseases \cite{vanDriel2006}. Similarly, $n_G=200$ features of $m_G=2020$ genes are extracted via spectral embedding of the \emph{HumanNet} gene network \cite{Lee2011}. Associations between diseases and genes are obtained from the database \emph{DisGeNET} \cite{Pinero2016}, which assigns association scores between $0$ and $1$ (larger scores mean stronger associations) to $40854$ gene-disease pairs (among the $2358\times 2020$ pairs we study in this experiment). For simplicity, we set the scores not present in \emph{DisGeNET} to $0$, an approach also adopted by prior work \cite{Natarajan2014}. 

We randomly partition the $2358$ diseases into two sets. All association scores of the first set of $2122$ diseases are observed, constituting the training set $\Omega_\mathrm{obs}$. All association scores of the other $236$ diseases are unobserved, constituting the test set $\Omega_\mathrm{uno}$. One such random partition is used to tune the hyperparameters of the learning models. Another $20$ random partitions are created, over which the performances of the models are evaluated. For any disease in the test set, we retrieve the top $k$ genes based on their predicted association scores, and compute recall at $k$, i.e., the number of correctly retrieved genes divided by all genes with positive association scores \cite{Natarajan2014}.

\textbf{Competing Methods.} We compare the proposed JDR+KR scheme to IMC \cite{Natarajan2014}, which is a state-of-the-art method for gene-disease association prediction, based on bilinear regression. 
Furthermore, to demonstrate the advantage of \emph{JDR}, we also embed gene and disease features for kernel regression using other DR methods: (a) We apply \emph{PCA} and \emph{pHd} \cite{Li1992}, which represent unsupervised and supervised DR methods respectively, and compute $r$-dimensional embeddings for $a$ and $b$ separately; (b) We apply pHd to find a $2r$-dimensional embedding for the concatenated vector $[a^T,b^T]^T$, dubbed \emph{cpHd}; (c) We apply sOLS \cite{Liu2016}, which represents existing DR methods that learns separable embeddings.

Next, we compare the computational cost of different approaches. The training procedure of the parametric model IMC is more expensive than that of a non-parametric approach, which computes the embedding matrices. On the other hand, the inference procedure of predicting an unknown response using IMC is very efficient (only involving an inner product between vectors of length $r$). Inference using kernel regression requires multiple evaluations of the kernel, and is more expensive. In particular, the number of kernel evaluations required to predict one response $\hat{y}_{i,j}$ is $m_D + m_G$ for separable DR methods (JDR, PCA, pHd, and sOLS), and is $m_Dm_G$ for cpHd.

\textbf{Prediction Results.} Figure \ref{fig:linebar} shows the prediction error of different methods. The plot on the left shows the mean recall at $k=10,20,\dots,100$, for a fixed $r=80$. 
The bar graph on the right shows the mean and standard deviation (error bar) of the recall at $k=100$, for $r=20, 40,\dots,100$. (We exclude cpHd from the bar graph, due to its extremely expensive inference procedure.)

The prediction performance of the proposed JDR+KR approach is significantly better than the bilinear regression approach IMC, especially when the reduced feature dimension $r$ is between $40$ and $100$. The recall of top $k$ genes obtained using our approach is consistently higher than IMC at different $k$. Our experiment demonstrates that when appropriate features correlated with the responses are extracted, non-parametric learning can outperform conventional parametric models. Dimensionality reduction improves the computational efficiency of non-parametric methods, and prevents overfitting.
Furthermore, JDR is crucial to the success of our approach. Replacing JDR with PCA, pHd, cpHd, or sOLS leads to performances that are comparable to or worse than IMC. This means the features extracted by JDR are indeed highly correlated with association scores between diseases and gens, and hence verifies the efficacy of the JDR algorithm.

\bibliographystyle{myIEEEtran}
\bibliography{nips2017}

\newpage

\begin{center}
\Large
Joint Dimensionality Reduction\\
for Separable Embedding Estimation\\
\textbf{Supplementary Material}
\end{center}

\section{Performance of Algorithm \ref{alg:jdr} under Non-Ideal Conditions}

In this section, we test how Algorithm \ref{alg:jdr} performs when the assumptions in Theorem \ref{thm:lin} are violated, i.e., when 1) sample means and covariances are used in data normalization, or 2) the entries of $a_i,b_i$ are i.i.d. following a uniform distribution on $[-\sqrt{3},\sqrt{3}]$, or 3) the entries of $a_i,b_i$ are i.i.d. following a Poisson distribution ($\lambda = 4$, normalized with zero mean and unit variance), or 4) $a_i,b_i$ are jointly Gaussian and slightly correlated (not independent). For the bilinear model, the log-log plots (natural logarithm) of mean NSEE versus $m$ are shown in Figure \ref{fig:violate}. There is no significant change in the performance, hence our algorithms are robust against slight violations of ideal assumptions.

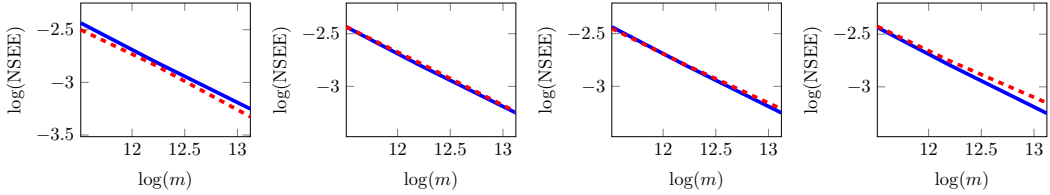
\begin{figure}[htbp]
\centering
%
%
\begin{tikzpicture}[scale = 0.7]

\begin{axis}[%
width=1.268in,
height=1in,
at={(0in,0in)},
scale only axis,
xmin=11.5129254649702,
xmax=13.1223633774043,
xlabel={$\log(m)$},
ymin=-3.51575278132376,
ymax=-2.24637352458138,
ylabel={$\log(\mathrm{NSEE})$},
axis background/.style={fill=white}
]
\addplot [color=blue,solid,line width=2.0pt,forget plot]
  table[row sep=crcr]{%
11.5129254649702	-2.43563598165736\\
12.2060726455302	-2.79689097157546\\
13.1223633774043	-3.25176321291227\\
};
\addplot [color=red,dashed,line width=2.0pt,forget plot]
  table[row sep=crcr]{%
11.5129254649702	-2.50058616160947\\
12.2060726455302	-2.82956987009585\\
13.1223633774043	-3.32649032424779\\
};
\end{axis}
\end{tikzpicture}%
~
%
%
\begin{tikzpicture}[scale = 0.7]

\begin{axis}[%
width=1.268in,
height=1in,
at={(0in,0in)},
scale only axis,
xmin=11.5129254649702,
xmax=13.1223633774043,
xlabel={$\log(m)$},
ymin=-3.47801872363627,
ymax=-2.20863946689389,
ylabel={$\log(\mathrm{NSEE})$},
axis background/.style={fill=white}
]
\addplot [color=blue,solid,line width=2.0pt,forget plot]
  table[row sep=crcr]{%
11.5129254649702	-2.43563598165736\\
12.2060726455302	-2.79689097157546\\
13.1223633774043	-3.25176321291227\\
};
\addplot [color=red,dashed,line width=2.0pt,forget plot]
  table[row sep=crcr]{%
11.5129254649702	-2.43489497761789\\
12.2060726455302	-2.77584270005024\\
13.1223633774043	-3.24120683231683\\
};
\end{axis}
\end{tikzpicture}%
~
%
%
\begin{tikzpicture}[scale = 0.7]

\begin{axis}[%
width=1.268in,
height=1in,
at={(0in,0in)},
scale only axis,
xmin=11.5129254649702,
xmax=13.1223633774043,
xlabel={$\log(m)$},
ymin=-3.478389225656,
ymax=-2.20900996891362,
ylabel={$\log(\mathrm{NSEE})$},
axis background/.style={fill=white}
]
\addplot [color=blue,solid,line width=2.0pt,forget plot]
  table[row sep=crcr]{%
11.5129254649702	-2.43563598165736\\
12.2060726455302	-2.79689097157546\\
13.1223633774043	-3.25176321291227\\
};
\addplot [color=red,dashed,line width=2.0pt,forget plot]
  table[row sep=crcr]{%
11.5129254649702	-2.45040266370146\\
12.2060726455302	-2.79241663529589\\
13.1223633774043	-3.21464610661094\\
};
\end{axis}
\end{tikzpicture}%
~
%
%
\begin{tikzpicture}[scale = 0.7]

\begin{axis}[%
width=1.268in,
height=1in,
at={(0in,0in)},
scale only axis,
xmin=11.5129254649702,
xmax=13.1223633774043,
xlabel={$\log(m)$},
ymin=-3.47333476862539,
ymax=-2.20395551188301,
ylabel={$\log(\mathrm{NSEE})$},
axis background/.style={fill=white}
]
\addplot [color=blue,solid,line width=2.0pt,forget plot]
  table[row sep=crcr]{%
11.5129254649702	-2.43563598165736\\
12.2060726455302	-2.79689097157546\\
13.1223633774043	-3.25176321291227\\
};
\addplot [color=red,dashed,line width=2.0pt,forget plot]
  table[row sep=crcr]{%
11.5129254649702	-2.42552706759612\\
12.2060726455302	-2.75601536346579\\
13.1223633774043	-3.15223244932912\\
};
\end{axis}
\end{tikzpicture}%
\caption{Log-log plots when ideal assumptions are violated. The blue solid lines are the performances when all the assumptions are met. The red dashed lines are the performances when some assumption is violated: 1) sample means and variances are used; 2) $a_i,b_i$ follow uniform distribution; 3) $a_i,b_i$ follow Poisson distribution; 4) $a_i$ and $b_i$ are weakly correlated.
}
\label{fig:violate}
\end{figure}

\section{Additional Results on Gene-Disease Association}

\textbf{Time Complexity of Different Approaches.} In our experiments, the numbers $m_D$ and $m_G$ of diseases and genes (a few thousands) are larger than the numbers $n_D$ and $n_G$ of disease and gene features (a few hundreds). The computational cost of the three stages (computing the embedding matrices, computing the low-dimensional features, and predicting one association score) of different approaches are summarized in Table \ref{tab:cost}. Since evaluating the kernel in kernel regression, although has a time complexity of $O(r)$, is usually much more expensive than $r$ floating point number additions and multiplications in matrix (vector) product, we quantify the cost of predicting one association score using the number of kernel evaluations. Clearly the embedding estimation of IMC is more expensive, but its inference is much faster, which only involves an inner product. Among different DR+KR approaches, the separable embeddings (JDR, PCA, pHd, and sOLS) are more efficient than cpHd, in terms of computing the embedded features and making new predictions.  

\begin{table}[ht]%
\newcolumntype{L}[1]{>{\raggedright\let\newline\\\arraybackslash}m{#1}}
\begin{tabular}{L{1.8cm} L{4cm} L{3.5cm} L{2.8cm}}
\hline
Approach & Estimating Embeddings$^*$ & Computing Features & Predicting One Association Score \\
\hline
JDR + KR & $O(m_Dm_G\min\{n_D, n_G\}+m_Dn_D^2+m_Gn_G^2)$ & $O(m_Dn_Dr+m_Gn_Gr)$ & $m_D + m_G$ \textbf{kernel evaluations} \\
PCA + KR & $O(m_Dn_D^2+m_Gn_G^2)$ & $O(m_Dn_Dr+m_Gn_Gr)$ & $m_D + m_G$ \textbf{kernel evaluations} \\
pHd + KR & $O(m_Dm_G+m_Dn_D^2+m_Gn_G^2)$ & $O(m_Dn_Dr+m_Gn_Gr)$ & $m_D + m_G$ \textbf{kernel evaluations} \\
sOLS + KR & $O(m_Dm_G+m_Dn_D^2+m_Gn_G^2)$ & $O(m_Dn_Dr+m_Gn_Gr)$ & $m_D + m_G$ \textbf{kernel evaluations} \\
cpHd + KR & $O(m_Dm_G\min\{n_D, n_G\}+m_Dn_D^2+m_Gn_G^2)$ & $O(m_Dm_Gn_Dr+m_Dm_Gn_Gr)$ & $m_D m_G$ \textbf{kernel evaluations} \\
IMC & $O(m_Dm_Gn_Dr+m_Dm_Gn_Gr)$ \textbf{per iteration} & $O(m_Dn_Dr+m_Gn_Gr)$ & $1$ \textbf{inner product} \\
\hline
\end{tabular}
\caption{The computational costs of different approaches of gene-disease association. $^*$In general, $m_Dm_G$ in the time complexities for embedding estimation can be replaced by the number of nonzero association scores that are observed. The table shows conservative bounds under the assumption that there exist $O(m_Dm_G)$ such association scores.}
\label{tab:cost}
\end{table}

\textbf{Recall Rates for Significant Associations.} In Section \ref{sec:gda}, we treat all nonzero association scores as positive when computing recall. It is also interesting to examine the recall rates for significant associations. The recall rates for gene-disease associations with scores larger than $0.01$ are shown in Figure \ref{fig:linebar2}. The advantage of using JDR+KR is even more pronounced in this case.

\begin{figure}[tb]
\centering
\subfigure{
\includegraphics[width=0.44\textwidth]{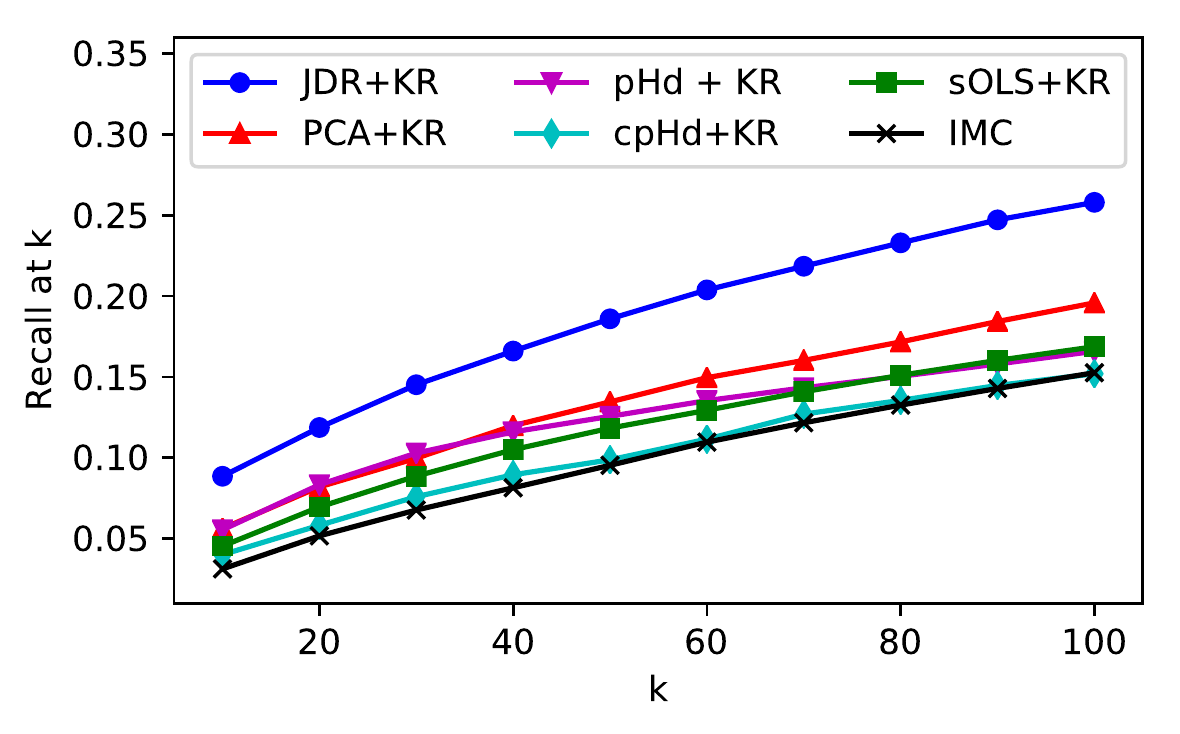}
}~~
\subfigure{
\includegraphics[width=0.44\textwidth]{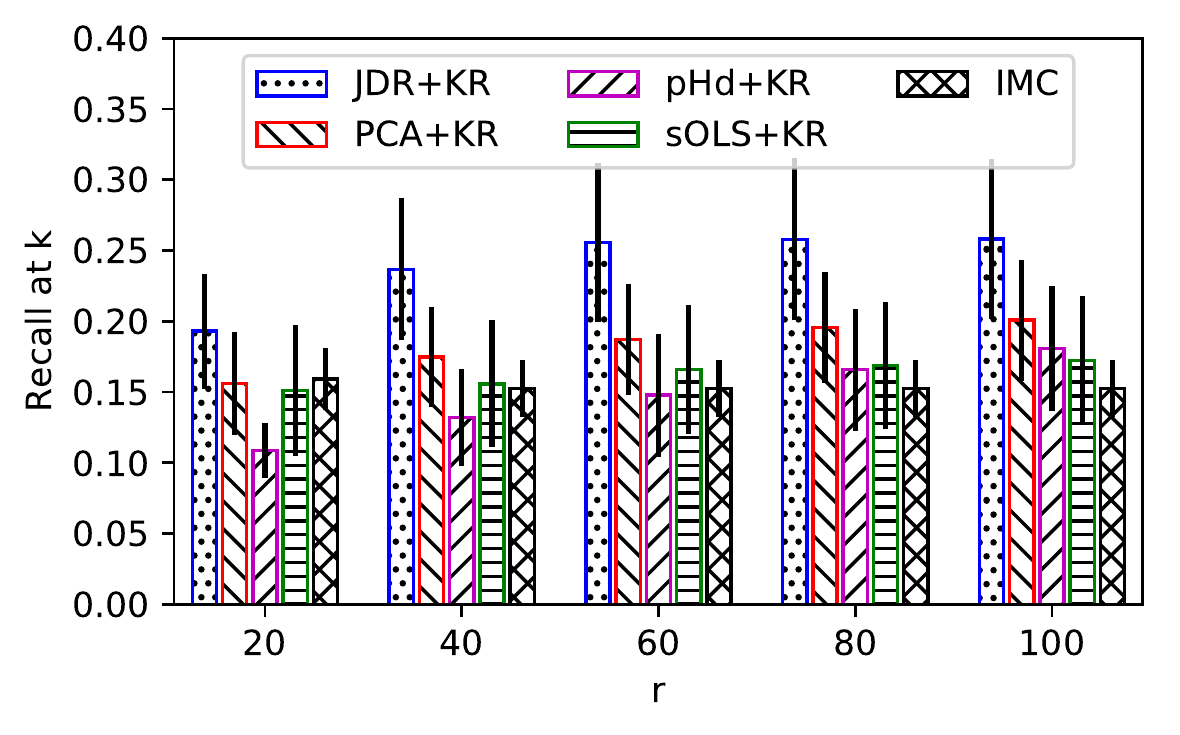}
}
\vspace{-0.15in}
\caption{Recall rates for significant associations of different methods.}
\vspace{-0.15in}
\label{fig:linebar2}
\end{figure}

\section{Proofs}

\subsection{Proof of Theorem \ref{thm:lin}} \label{sec:proof_lin}

Let $\tU\in\bbR^{n_1\times(n_1-r)}$ and $\tV\in\bbR^{n_2\times(n_2-r)}$ be matrices of orthonormal columns that satisfy $U^T\tU=0$, $V^T\tV = 0$, i.e., the columns of $\tU$ and $\tV$ span the orthogonal complements of the subspaces spanned by the columns of $U$ and $V$. Define $\ba_i \coloneqq U^T a_i$, $\ta_i \coloneqq \tU^T a_i$, $\bb_i \coloneqq V^T b_i$, and $\tb_i \coloneqq \tV^T b_i$.
\begin{lemma}\label{lem:r_ind}
$\{\ba_i\}_{i=1}^m$, $\{\ta_i\}_{i=1}^m$, $\{\bb_i\}_{i=1}^m$, and $\{\tb_i\}_{i=1}^m$ are all independent Gaussian random vectors. Moreover, $\ba_i\sim N(0,I_{r})$, $\ta_i\sim N(0,I_{n_1-r})$, $\bb_i\sim N(0,I_{r})$, $\tb_i\sim N(0,I_{n_2-r})$.
\end{lemma}

\begin{proof}[Proof of Lemma \ref{lem:r_ind}]
Obviously, these vectors are all zero mean Gaussian random vectors. Independence follows from two facts:
\begin{enumerate}
	\item $\{a_i\}_{i=1}^m$ and $\{b_i\}_{i=1}^m$ are independent Gaussian vectors. 
	\item $\bbE[\ba_i \ta_i^T] = \bbE [U^T a_i a_i^T \tU] = U^T I_{n_1} \tU = 0$, and $\bbE[\bb_i \tb_i^T] = \bbE [V^T b_i b_i^T \tV] = V^T I_{n_2} \tV = 0$. (Uncorrelated Gaussian random vector are independent.)
\end{enumerate}
Covariance matrices are easy to compute. For example, $\operatorname{Cov}(\ba_i) = \bbE[\ba_i \ba_i^T] = \bbE [U^T a_i a_i^T U] = U^T I_{n_1} U = I_r$.
\end{proof}

\begin{lemma} \label{lem:independent}
$y_i$ and $\ta_i,\tb_i$ are independent.
\end{lemma}
\begin{proof}[Proof of Lemma \ref{lem:independent}]
By Lemma \ref{lem:r_ind}, $\ta_i,\tb_i$ and $\ba_i,\bb_i$ are independent. By the Markov chain assumption \eqref{eq:markov}, $y_i$ and $\ta_i,\tb_i$ are conditionally independent given $\ba_i,\bb_i$. Therefore, by contraction property of conditional independence, $y_i$ and $\ta_i,\tb_i$ are independent.

When $y_i$ is a continuous random variable, the contraction property can be proved as follows:
\begin{align}
p(y_i,\ta_i,\tb_i) =&~ p(y_i,\ta_i,\tb_i | \ba_i,\bb_i)\cdot p(\ba_i,\bb_i) \nonumber\\
=&~ p(y_i| \ba_i,\bb_i) \cdot p(\ta_i,\tb_i | \ba_i,\bb_i)  \cdot p(\ba_i,\bb_i) \nonumber\\
=&~ p(y_i| \ba_i,\bb_i) \cdot p(\ta_i,\tb_i)  \cdot p(\ba_i,\bb_i) \nonumber\\
=&~ p(y_i) \cdot p(\ta_i,\tb_i). \nonumber
\end{align}
The second line follows from the conditional independence of $y_i$ and $(\ta_i,\tb_i)$ given $(\ba_i,\bb_i)$. The third line follows from the independence between $(\ta_i,\tb_i)$ and $(\ba_i,\bb_i)$.
\end{proof}

\textbf{Next, we prove the first part of Theorem \ref{thm:lin}.} Note that
\begin{align}
a_i y_i b_i^T = &~  (UU^T + \tU\tU^T)a_i y_i b_i^T (VV^T+\tV\tV^T)  \nonumber \\
= &~ U\ba_i y_i \bb_i^TV^T + \tU\ta_i y_i \tb_i^T\tV^T + U\ba_i y_i \tb_i^T\tV^T +\tU\ta_i y_i \bb_i^TV^T. \label{eq:split}
\end{align}
Define $Q \coloneqq \bbE\left[\ba_1 f(\ba_1,\bb_1) \bb_1^T \right]$, we have
\begin{align}
\bbE \left[ \hX_0 \right] = &~ \frac{1}{m} \sum_{i=1}^{m}\bbE\left[ a_i y_i b_i^T\right] \nonumber\\
= &~ \bbE\left[a_i y_i b_i^T \right] \nonumber\\
= &~ \bbE\left[U\ba_i y_i \bb_i^TV^T\right] + \bbE\left[\tU\ta_i y_i \tb_i^T\tV^T\right] + \bbE\left[U\ba_i y_i \tb_i^T\tV^T\right] + \bbE\left[\tU\ta_i y_i \bb_i^TV^T\right] \nonumber\\
= &~ U\bbE\left[\ba_i y_i \bb_i^T\right]V^T + \tU\bbE\left[\ta_i\right]  \bbE\left[y_i\right] \bbE\left[\tb_i^T\right]\tV^T + U\bbE\left[\ba_i y_i\right] \bbE\left[\tb_i^T\right]\tV^T + \tU\bbE\left[\ta_i\right]\bbE\left[ y_i \bb_i^T\right]V^T \nonumber\\
=&~  UQV^T. \nonumber
\end{align}
The second line follows from independence of $y_i,\ta_i,\tb_i$ (see Lemma \ref{lem:independent}), and the third line follows from the tower property of conditional expectation:
\begin{align*}
\bbE\left[\ba_i y_i \bb_i^T \right] = &~ \bbE\left[\ba_i ~\bbE[ y_i | a_i,b_i ]~\bb_i^T \right] = \bbE\left[\ba_i f(\ba_i,\bb_i)\bb_i^T \right] = Q.
\end{align*}

Next we show that $\hX_0$ is a consistent estimator of $UQV^T$ by bounding their difference. Thanks to the finite second moments assumptions, we define the following notations:
\begin{align*}
& \sigma_{y|a,b}^2 \coloneqq \var[y|a,b], \quad \sigma^2 \coloneqq \bbE\left[\norm{\ba_1 f(\ba_1,\bb_1) \bb_1^T-Q}_\rmF^2 \right],\\ 
& \tau_0^2 \coloneqq \bbE\left[\left|f(\ba_1,\bb_1)\right|^2\right], \quad \tau_1^2 \coloneqq \bbE\left[\norm{\ba_1f(\ba_1,\bb_1)}_2^2\right], \quad \tau_2^2 \coloneqq \bbE\left[\norm{\bb_1f(\ba_1,\bb_1)}_2^2 \right].
\end{align*}

\begin{lemma}\label{lem:mean_var}
\[
\bbE\left[\norm{\ba_i y_i \bb_i^T-Q}_\rmF^2 \right] = r^2\sigma_{y|a,b}^2 + \sigma^2, \quad \bbE\left[\left|y_i\right|^2\right] = \sigma_{y|a,b}^2+ \tau_0^2,
\]
\[
\bbE\left[\norm{\ba_iy_i}_2^2\right] = r\sigma_{y|a,b}^2+\tau_1^2, \quad \bbE\left[\norm{\bb_iy_i}_2^2 \right] = r\sigma_{y|a,b}^2+\tau_2^2.
\]
\end{lemma}

\begin{proof}[Proof of Lemma \ref{lem:mean_var}]
For the first equality, note that
\begin{align*}
\norm{\ba_i y_i \bb_i^T-Q}_\rmF^2 =&~ \norm{\ba_i \left[y_i-f(\ba_i,\bb_i)\right] \bb_i^T}_\rmF^2 + 2\left<\ba_i \left[y_i-f(\ba_i,\bb_i)\right] \bb_i^T,\ba_i f(\ba_i,\bb_i) \bb_i^T-Q\right>\\
&~ + \norm{\ba_i f(\ba_i,\bb_i) \bb_i^T-Q}_\rmF^2 \\
=&~ |y_i-f(\ba_i,\bb_i)|^2 \cdot\norm{\ba_i}_2^2\cdot \norm{\bb_i}_2^2 + 2 \left[y_i-f(\ba_i,\bb_i)\right] \cdot \left<\ba_i \bb_i^T,\ba_i f(\ba_i,\bb_i) \bb_i^T-Q\right> \\
&~ + \norm{\ba_i f(\ba_i,\bb_i) \bb_i^T-Q}_\rmF^2.
\end{align*}
Hence we have
\begin{align*}
\bbE\left[\norm{\ba_i y_i \bb_i^T-Q}_\rmF^2~ \middle| ~a,b \right] = \sigma_{y|a,b}^2\cdot\norm{\ba_i}_2^2\cdot \norm{\bb_i}_2^2 + \norm{\ba_i f(\ba_i,\bb_i) \bb_i^T-Q}_\rmF^2.
\end{align*}
Therefore,
\begin{align*}
\bbE\left[ \norm{\ba_i y_i \bb_i^T-Q}_\rmF^2 \right] = &~ \bbE\left[ \bbE\left[ \norm{\ba_i y_i \bb_i^T-Q}_\rmF^2~ \middle|~ a,b \right] \right] \\
= &~ \sigma_{y|a,b}^2\cdot\bbE\left[\norm{\ba_i}_2^2\right] \cdot \bbE\left[\norm{\bb_i}_2^2\right] + \bbE\left[\norm{\ba_i f(\ba_i,\bb_i) \bb_i^T-Q}_\rmF^2\right]\\
= &~ r^2\sigma_{y|a,b}^2 + \sigma^2.
\end{align*}
The other equalities can be proved similarly.
\end{proof}

Note that
\begin{align*}
\norm{a_i y_i b_i^T -UQV^T}_\rmF^2  = &~ \norm{U(\ba_i y_i \bb_i^T-Q)V^T}_\rmF^2 + \norm{\tU\ta_i y_i \tb_i^T\tV^T}_\rmF^2  + \norm{ U\ba_i y_i \tb_i^T\tV^T}_\rmF^2 + \norm{\tU\ta_i y_i \bb_i^TV^T}_\rmF^2\\
= &~ \norm{\ba_i y_i \bb_i^T-Q}_\rmF^2 + \norm{\ta_i y_i \tb_i^T}_\rmF^2 + \norm{\ba_i y_i \tb_i^T}_\rmF^2 + \norm{\ta_i y_i \bb_i^T}_\rmF^2\\
= &~ \norm{\ba_i y_i \bb_i^T-Q}_\rmF^2 + \norm{\ta_i}_2^2 |y_i|^2 \norm{\tb_i}_2^2 + \norm{\ba_i y_i}_2^2 \norm{\tb_i}_2^2 + \norm{\ta_i}_2^2 \norm{\bb_i y_i}_2^2.
\end{align*}
where the first line follows from Pythagorean theorem, the second line follows from $\norm{U\Sigma V^T}_\rmF = \norm{\Sigma}_\rmF$ for matrices $U,V$ of orthonormal columns, and the third line follows from $\norm{ab^T}_\rmF = \norm{a}_2\norm{b}_2$.
By Lemma \ref{lem:mean_var},
\begin{align}
\bbE\left[\norm{a_i y_i b_i^T - UQV^T}_\rmF^2 \right] \leq &~ (r^2\sigma_{y|a,b}^2+\sigma^2) + (n_1-r)(n_2-r)(\sigma_{y|a,b}^2+\tau_0^2) \nonumber\\
&~ + (n_2-r)(r\sigma_{y|a,b}^2+\tau_1^2) + (n_1-r)(r\sigma_{y|a,b}^2+\tau_2^2) \nonumber\\
= &~ n_1n_2\sigma_{y|a,b}^2 + \sigma^2+(n_1-r)(n_2-r)\tau_0^2+(n_2-r)\tau_1^2+ (n_1-r)\tau_2^2. \nonumber
\end{align}
It follows from the independence between $\{a_i y_i b_i^T\}_{i=1}^m$ that
\begin{align*}
\bbE \left[ \norm{\hX_0- UQV^T}_\rmF^2 \right] = &~ \frac{1}{m^2} \sum_{i=1}^{m}\bbE\left[\norm{ a_i y_i b_i^T - UQV^T}_\rmF^2 \right] \\
\leq &~ \frac{n_1n_2\sigma_{y|a,b}^2 + \sigma^2+(n_1-r)(n_2-r)\tau_0^2+(n_2-r)\tau_1^2+ (n_1-r)\tau_2^2}{m}.
\end{align*}
Clearly, $\sigma_{y|a,b},\sigma,\tau_0,\tau_1,\tau_2$ are all independent of $n_1$, $n_2$, and $m$. Since $r=O(1)$,
\[
\bbE \left[ \norm{\hX_0- UQV^T}_\rmF^2 \right] = O\left(\frac{n_1n_2}{m}\right).
\]

By triangle inequality, and the fact that $\hU\hSigma\hV^T$ is the best rank-$r$ approximation of $\hX_0$ via compact SVD,
\begin{align*}
\norm{\hU\hSigma\hV^T-UQV^T}_\rmF \leq  \norm{\hU\hSigma\hV^T-\hX_0}_\rmF+\norm{\hX_0-UQV^T}_\rmF \leq 2\norm{\hX_0-UQV^T}_\rmF.
\end{align*}
\textbf{The error bound in the second part of Theorem \ref{thm:lin}} is given by:
\begin{align*}
\max\left\{\bbE\left[d(U,\hU)\right],~ \bbE\left[d(V,\hV)\right]\right\} & \leq \bbE\left[\max\left\{d(U,\hU),~ d(V,\hV)\right\} \right] \\
& \leq \bbE\left[ \frac{1}{\sigma_r}\norm{UQV^T-\hU\hSigma\hV^T}_\rmF \right] \\
& \leq \frac{2}{\sigma_r} \bbE\left[\norm{\hX_0 - UQV^T}_\rmF\right] \\
& \leq \frac{2}{\sigma_r}\sqrt{\bbE\left[\norm{\hX_0 - UQV^T}_\rmF^2\right]} = O\left(\sqrt{\frac{n_1n_2}{m}}\right).
\end{align*}
The first and the last lines follow from Jensen's inequality. The second line follows from Lemma \ref{lem:subspace}, where $\sigma_r$ is the smallest singular value of $Q$.

\begin{lemma} \label{lem:subspace}
Let $\sigma_r$ denote the smallest singular value of $Q$, which is nonzero when $Q$ is non-singular.
\begin{align*}
\max\left\{ d(U,\hU),~d(V,\hV) \right\} \leq \frac{1}{\sigma_r}\norm{UQV^T-\hU\hSigma\hV^T}_\rmF.
\end{align*}
\end{lemma}
\begin{proof}[Proof of Lemma \ref{lem:subspace}]
We only prove the bound for $d(U,\hU)$. The bound for $d(V,\hV)$ can be proved similarly. 

Let $\widetilde{\hU}\in\bbR^{n_1\times (n_1-r)}$ denote a matrix of orthonormal columns that satisfies $\widetilde{\hU}^T\hU = 0$, then we have the following two identities:
\begin{align*}
&\norm{\tU^T\hU}_\rmF^2 = \norm{\tU\tU^T\hU}_\rmF^2 = \norm{\hU}_\rmF^2-\norm{UU^T\hU}_\rmF^2 = r- \norm{U^T\hU}_\rmF^2,\\
&\norm{\widetilde{\hU}^T U}_\rmF^2 = \norm{\widetilde{\hU}\widetilde{\hU}^TU}_\rmF^2 = \norm{U}_\rmF^2-\norm{\hU\hU^TU}_\rmF^2 = r- \norm{\hU^TU}_\rmF^2.
\end{align*}
It follows that
\[
d(U,\hU) = \norm{\hU - UU^T \hU}_\rmF = \norm{\tU^T\hU}_\rmF = \norm{\widetilde{\hU}^TU}_\rmF = \norm{U - \hU\hU^TU}_\rmF = d(\hU, U).
\]

Therefore,
\begin{align}
d(U,\hU) = &~\norm{\widetilde{\hU}^T U}_\rmF \leq \frac{1}{\sigma_r} \norm{\widetilde{\hU}^T UQV^T}_\rmF = \frac{1}{\sigma_r} \norm{\widetilde{\hU}^T (UQV^T-\hU\hSigma\hV^T)}_\rmF \nonumber\\
\leq &~ \frac{1}{\sigma_r} \norm{\widetilde{\hU}^T}_2 \norm{UQV^T-\hU\hSigma\hV^T}_\rmF = \frac{1}{\sigma_r}  \norm{UQV^T-\hU\hSigma\hV^T}_\rmF. \nonumber
\end{align}

\end{proof}

\subsection{Proof of Theorem \ref{thm:optimal}} \label{sec:proof_optimal}

Suppose $\Omega\subset\bbR^{n_1\times n_2}$ is a cone, and $\calB_{n_1\times n_2}\subset\bbR^{n_1\times n_2}$ is the unit ball centered at the origin in the Frobenius norm. Define
\[
\Delta \coloneqq (\Omega-\Omega) \bigcap \calB_{n_1\times n_2},
\]
\[
\norm{Y}_{\Delta^\circ} \coloneqq \sup\limits_{X\in\Delta} \left<Y,X\right>.
\]
Here, ${\Delta^\circ}$ is the polar set of $\Delta$. We restate some results by Plan et al. \cite{Plan2016} in Lemmas \ref{lem:seminorm} -- \ref{lem:max2}. Lemma \ref{lem:seminorm} follows from the properties of polar sets.
\begin{lemma} \label{lem:seminorm}
For symmetric set $\Delta$, $\norm{\cdot}_{\Delta^\circ}$ is a pseudo-norm, or equivalently
\begin{enumerate}
	\item $\norm{Y}_{\Delta^\circ} \geq 0$, and $\norm{0}_{\Delta^\circ} = 0$.
	\item $\norm{cY}_{\Delta^\circ} = |c| \cdot\norm{Y}_{\Delta^\circ}$.
	\item $\norm{Y_1+Y_2}_{\Delta^\circ} \leq \norm{Y_1}_{\Delta^\circ} + \norm{Y_2}_{\Delta^\circ} $.
\end{enumerate}
Properties 2 and 3 imply that $\norm{\cdot}_{\Delta^\circ}$ is convex.
\end{lemma}

\begin{lemma} \label{lem:proj}
If $\Omega\subset\bbR^{n_1\times n_2}$ is a cone, and $X\in\Omega$, then for $\hX_0\in \bbR^{n_1\times n_2}$,
\[
\norm{P_\Omega \hX_0 - X}_\rmF \leq 2\norm{\hX_0 - X}_{\Delta^\circ}
\]
\end{lemma}

\begin{proof}
Since $\Omega$ is a cone, we have $(\Omega-\Omega)\bigcap t\calB_{n_1\times n_2} = t\Delta$ for $t>0$. Moreover,
\begin{align*}
\frac{1}{t}\norm{Y}_{(t\Delta)^\circ} = \frac{1}{t}\sup\limits_{X\in t\Delta} \left<Y,X\right> = \sup\limits_{X\in\Delta} \left<Y,X\right> = \norm{Y}_{\Delta^\circ}.
\end{align*}
By \cite[Corollary 8.3]{Plan2016}, for every $t>0$ we have
\[
\norm{P_{\Omega} \hX_0 -X }_\rmF \leq \max\left\{t, ~\frac{2}{t}\norm{\hX_0 - X}_{(t\Delta)^\circ}\right\} = \max\left\{t,~2\norm{\hX_0 - X}_{\Delta^\circ}\right\}.
\]
Lemma \ref{lem:proj} follows from letting $t$ go to $0$.
\end{proof}

\begin{lemma} \label{lem:max2}
Suppose $\Omega_{r}\subset \bbR^{n_1\times n_2}$ is the set of matrices of at most rank $r$, and $\Delta_{r} = (\Omega_{r}-\Omega_{r}) \bigcap \calB_{n_1\times n_2}$. Then
\[
\norm{Y}_{\Delta_{r}^\circ} \leq \min\left\{\norm{Y}_\rmF,~\sqrt{2r} \norm{Y}\right\}.
\]
\end{lemma}
\begin{proof}
By Cauchy-Schwarz inequality,
\begin{align*}
\norm{Y}_{\Delta_r^\circ} = \sup\limits_{X\in\Delta_r} \left<Y,X\right> \leq \sup\limits_{X\in\Delta_r} \norm{X}_\rmF\norm{Y}_\rmF = \norm{Y}_\rmF.
\end{align*}
Since
\[
\Delta_{r} \subset \{X\in\bbR^{n_1\times n_2}: \rank(X)\leq 2r, \norm{X}_\rmF\leq 1\} \subset \{X\in\bbR^{n_1\times n_2}: \norm{X}_*\leq \sqrt{2r}\} \eqqcolon \Delta_*,
\]
By H\"{o}lder's inequality,
\begin{align*}
\norm{Y}_{\Delta_{r}^\circ} = \sup\limits_{X\in\Delta_{r}} \left<Y,X\right>\leq \sup\limits_{X\in\Delta_*} \left<Y,X\right> \leq \sup\limits_{X\in\Delta_*} \norm{X}_*\norm{Y} = \sqrt{2r}\norm{Y}.
\end{align*}
\end{proof}

\begin{lemma} \label{lem:degen_gauss}
Suppose $u\sim N(0,I_n)$, $\tilde{u}\sim N(0,P)$ and $P\in\bbR^{n\times n}$ is a projection matrix. Then for a convex function $g(\cdot)$, we have $\bbE[g(\tilde{u})]\leq \bbE[g(u)]$.
\end{lemma}
\begin{proof}
Let $\bar{u} \sim  N(0,I-P)$ be independent from $\tilde{u}$, then $\tilde{u}+\bar{u}$ have the same distribution as $u$.
\[
\bbE[g(\tilde{u})] = \bbE[g(\tilde{u}+ \bbE[\bar{u}])] \leq \bbE[g(\tilde{u}+ \bar{u})] = \bbE[g(u)],
\]
where the inequality follows from Jensen's inequality.
\end{proof}

\begin{lemma} \label{lem:use_exp}
If $\{y_i\}_{i=1}^m$ are i.i.d. and satisfy the light-tailed response condition (see Definition \ref{def:light}), then
\begin{align*}
\bbP\left[\max_i |y_i| > t \log m\right] \leq C m^{1-ct}.
\end{align*}
\end{lemma}
\begin{proof}
\begin{align*}
\bbP\left[\max_i |y_i| > t \log m\right] \leq \sum_i \bbP\left[ |y_i| > t \log m\right] \leq m Ce^{-c t \log m} = C m^{1-ct}.
\end{align*}
\end{proof}

We need the following matrix Bernstein inequality.
\begin{lemma} \label{lem:bern}\cite[Theorem 6.2]{Tropp2011}
Suppose $\{X_i\}_{i=1}^m$ are $n\times n$ symmetric independent random matrices,
\[
\bbE[X_i] = 0,\quad  \bbE\left[X_i^k\right] \preceq \frac{k!}{2}\cdot R^{k-2}A_i^2,\quad \sigma_A^2 \coloneqq \norm{\sum_i A_i^2}. 
\] 
Then for all $t\geq 0$, we have
\[
\bbP\left[\lambda_{\max}\left(\sum_i X_i\right)\geq t \right] \leq n \cdot \exp\left(\frac{-t^2/2}{\sigma_A^2+Rt}\right).
\]
\end{lemma}

\textbf{Next, we prove Theorem \ref{thm:optimal}.} 
By \eqref{eq:split} and triangle inequality,
\begin{align}
\norm{\hX_0 - UQV^T}_{\Delta_{r}^\circ} \leq &~ \norm{ U \left(\frac{1}{m}\sum_{i=1}^m\ba_i y_i \bb_i^T -Q\right)V^T}_{\Delta_{r}^\circ} + \norm{\tU\left(\frac{1}{m}\sum_{i=1}^m\ta_i y_i \tb_i^T\right)\tV^T}_{\Delta_{r}^\circ}   \nonumber\\
&~ + \norm{U \left(\frac{1}{m}\sum_{i=1}^m\ba_i y_i \tb_i^T \right)\tV^T}_{\Delta_{r}^\circ} + \norm{\tU\left(\frac{1}{m}\sum_{i=1}^m\ta_i y_i \bb_i^T\right)V^T }_{\Delta_{r}^\circ} \nonumber\\
\eqqcolon &~ T_1+T_2+T_3+T_4. \label{eq:split3}
\end{align}

Next, we bound the expectation of the four terms. For $T_1$, we use Lemma \ref{lem:max2}:
\begin{align}
\bbE[T_1] \leq &~ \bbE\left[\norm{U \left(\frac{1}{m}\sum_{i=1}^m\ba_i y_i \bb_i^T -Q\right)V^T}_\rmF \right] \nonumber\\
= &~ \bbE\left[\norm{\left(\frac{1}{m}\sum_{i=1}^m\ba_i y_i \bb_i^T -Q\right)}_\rmF \right]  \nonumber\\
\leq &~ \sqrt{\bbE\left[\norm{\left(\frac{1}{m}\sum_{i=1}^m\ba_i y_i \bb_i^T -Q\right)}_\rmF^2 \right]} \nonumber\\
\leq &~ \sqrt{\frac{r^2\sigma_{y|a,b}^2+\sigma^2}{m}}. \label{eq:T1}
\end{align}
Similarly, one can obtain bounds on the expectations of $T_3$ and $T_4$:
\begin{align}
\bbE[T_3] \leq &~ \bbE\left[\norm{\frac{1}{m}\sum_{i=1}^m U\ba_i y_i v_i^T}_{\Delta_{r}^\circ}\right] \leq \frac{1}{m}~ \bbE\left[\norm{\sum_{i=1}^m (\ba_i y_i)v_i^T}_\rmF\right] \nonumber\\
\leq &~ \frac{1}{m}~ \sqrt{\bbE\left[\norm{\sum_{i=1}^m (\ba_i y_i)v_i^T}_\rmF^2\right]}\nonumber\\
\leq &~ \frac{1}{m}~ \sqrt{m \bbE\left[\norm{\ba_i y_i}_2^2\right] \bbE\left[\norm{v_i}_2^2\right]} \nonumber\\
\leq &~ \sqrt{\frac{n_2 (r\sigma_{y|a,b}^2+\tau_1^2)}{m}}, \nonumber\\
\bbE[T_4] \leq &~ \sqrt{\frac{n_1 (r\sigma_{y|a,b}^2+\tau_2^2)}{m}}. \nonumber
\end{align}

Suppose $u_i\sim N(0,I_{n_1})$, $v_i\sim N(0,I_{n_2})$, $\{u_i\}_{i=1}^m$, $\{v_i\}_{i=1}^m$, and $\{y_i\}_{i=1}^m$, $\{\ba_i\}_{i=1}^m$, $\{\bb_i\}_{i=1}^m$ are independent. Replacing $\tU\ta_i,\tV\tb_i$ in $T_2$ by $u_i,v_i$, by Lemmas \ref{lem:degen_gauss} and \ref{lem:max2},
\begin{align}
\bbE[T_2] \leq \bbE\left[\norm{\frac{1}{m}\sum_{i=1}^m u_i y_i v_i^T}_{\Delta_{r}^\circ}\right]\leq \frac{\sqrt{2r}}{m}~ \bbE\left[ \norm{\sum_{i=1}^m u_iy_i v_i^T}\right]. \label{eq:T2_second}
\end{align}

We give the following concentration of measure bound on the spectral norm in \eqref{eq:T2_second},
\begin{align}
&~ \bbP\left[\norm{\sum_{i=1}^m u_iy_i v_i^T} \geq t^2 \sqrt{(n_1+n_2)m}\cdot\log m\right] \nonumber\\
\leq &~ \bbP\left[\norm{\sum_{i=1}^m u_iy_i v_i^T} \geq t^2 \sqrt{(n_1+n_2)m}\cdot\log m,~\max_i |y_i| \leq t \log m\right] + \bbP\left[\max_i |y_i| > t \log m\right]  \nonumber\\
\leq &~ (n_1+n_2)\cdot \exp\left(\frac{-t^4}{2t^2+6t^3}\right)+Cm^{1-ct}. \label{eq:use_conc}
\end{align}
The bounds on the first and second terms follow from Lemmas \ref{lem:bern} and \ref{lem:use_exp}, respectively. The derivation for the first bound can be found in Section \ref{sec:spectral}.
By \eqref{eq:use_conc},
\begin{align*}
\bbP\left[\norm{\sum_{i=1}^m u_iy_i v_i^T} \geq t^2 \sqrt{(n_1+n_2)m}\cdot\log m\right]
\leq \begin{cases}
1, & \quad \text{if}~t\leq 8\log(n_1+n_2),\\
(n_1+n_2)\cdot \exp\left(\frac{-t}{8}\right)+2Cm^{1-ct}, & \quad \text{if}~t> 8\log(n_1+n_2).
\end{cases} 
\end{align*}
Hence
\begin{align*}
\bbE\left[ \norm{\sum_{i=1}^m u_iy_i v_i^T}\right] = &~ \int_{0}^{\infty} \bbP\left[\norm{\sum_{i=1}^m u_iy_i v_i^T} \geq x\right] \rmd x\\
\leq &~ \sqrt{(n_1+n_2)m}\cdot \log m\cdot \Big(64\log^2(n_1+n_2)+128\log(n_1+n_2)+128\\
&~ +\frac{16C \log(n_1+n_2)}{c\log m\cdot m^{8c\log(n_1+n_2)-1}}+\frac{2C}{c^2 \log^2m \cdot m^{8c\log(n_1+n_2)-1}} \Big)\\
\leq &~ 256(C+2) \sqrt{(n_1+n_2)m}\cdot \log m \cdot\log^2(n_1+n_2).
\end{align*}
The derivation is tedious but elementary, in which the assumptions $c>\frac{1}{8\log(n_1+n_2)}$ and $m>n_1+n_2$ are invoked.
By \eqref{eq:T2_second},
\begin{align*}
\bbE[T_2] \leq  \frac{\sqrt{2r}}{m}~ \bbE\left[ \norm{\sum_{i=1}^m u_iy_i v_i^T} \right] \leq 256\sqrt{2}(C+2) \sqrt{\frac{(n_1+n_2)r\log^2m\log^4(n_1+n_2)}{m}}
\end{align*}

By Lemma \ref{lem:proj} and \eqref{eq:split3}, we have
\begin{align*}
\bbE\left[\norm{\hU\hSigma\hV^T - UQV^T}_\rmF\right] \leq &~ 2 \bbE\left[\norm{\hX_0 - UQV^T}_{\Delta_{r}^\circ}\right] \leq 2 \bbE[T_1] + 2 \bbE[T_2] + 2 \bbE[T_3] + 2 \bbE[T_4] \\
\leq &~ 2\sqrt{\frac{r^2\sigma_{y|a,b}^2+\sigma^2}{m}} + 512\sqrt{2}(C+2) \sqrt{\frac{(n_1+n_2)r\log^2m\log^4(n_1+n_2)}{m}} \\
&~ + 2\sqrt{\frac{n_2 (r\sigma_{y|a,b}^2+\tau_1^2)}{m}} + 2\sqrt{\frac{n_1 (r\sigma_{y|a,b}^2+\tau_2^2)}{m}} \\
= &~ O\left(\sqrt{\frac{(n_1+n_2)\log^2m\log^4(n_1+n_2)}{m}}\right),
\end{align*}
where the last line uses the assumption that $r=O(1)$. Therefore, Theorem \ref{thm:optimal} follows from Lemma \ref{lem:subspace}.


\subsubsection{Spectral Norm Bound in the Proof of Theorem \ref{thm:optimal}}\label{sec:spectral}
In this section, we prove the first bound in \eqref{eq:use_conc}. We have
\begin{align}
&~ \bbP\left[\norm{\sum_{i=1}^m u_iy_i v_i^T} \geq t^2 \sqrt{(n_1+n_2)m}\cdot\log m,~\max_i |y_i| \leq t \log m~\middle |~\{y_i\}_{i=1}^m \right] \nonumber\\
= &~ \bbP\left[\norm{\sum_{i=1}^m u_iy_i v_i^T} \geq t^2 \sqrt{(n_1+n_2)m}\cdot\log m~\middle |~\{y_i\}_{i=1}^m \right] \cdot \ind{\max_i |y_i| \leq t \log m} \nonumber\\
\leq &~ (n_1+n_2)\cdot \exp\left(\frac{-t^4 (n_1+n_2)m\log^2 m /2}{t^2(n_1+n_2)m\log^2 m+et^3 (n_1+n_2)m\log^2 m}\right) \label{eq:use_bern}\\
\leq &~ (n_1+n_2)\cdot \exp\left(\frac{-t^4}{2t^2+6t^3}\right), \nonumber
\end{align}
Next, we show how \eqref{eq:use_bern} follows from the matrix Bernstein inequality in Lemma \ref{lem:bern}. 
The rest of the derivation is conditioned on $\{y_i\}_{i=1}^m$ that satisfy $\max_i |y_i| \leq t \log m$, hence $\sum_i y_i^2 \leq t^2 m (\log m)^2$. Define $(n_1+n_2)\times (n_1+n_2)$ matrices ($i=1,2,\cdots,m$):
\[
X_i = \begin{bmatrix}
0 & u_iy_i v_i^T\\
v_iy_i u_i^T & 0
\end{bmatrix}.
\]
They satisfy
\[
\lambda_{\max}\left(\sum_i X_i\right) = \norm{\sum_i u_i y_i v_i^T},
\]
\[
\bbE[X_i] = 0,\quad \bbE\left[X_i^k\right]  = 0, \quad \text{if $k$ is odd,}
\]
\begin{align*}
\bbE\left[X_i^k\right] 
= &~ y_i^k (n_1+2)\cdots(n_1+k-2)(n_2+2)\cdots(n_2+k-2)\begin{bmatrix}
n_2I_{n_1} & 0\\
0 & n_1I_{n_2}
\end{bmatrix}\\
\preceq &~ \frac{k!}{2}\left[e(n_1+n_2)\max_i|y_i|\right]^{k-2} \begin{bmatrix}
y_i^2 n_2I_{n_1} & 0\\
0 & y_i^2 n_1I_{n_2}
\end{bmatrix},  \quad \text{if $k$ is even.}
\end{align*}

Let $R= e(n_1+n_2)\max_i|y_i| \leq et (n_1+n_2)\log m\leq et \sqrt{(n_1+n_2)m}\cdot \log m$,
\[
A_i^2 = \begin{bmatrix}
y_i^2 n_2I_{n_1} & 0\\
0 & y_i^2 n_1I_{n_2}
\end{bmatrix},
\]
and $\sigma_A^2 = \sum_i y_i^2 \max\{n_1,n_2\}\leq t^2(n_1+n_2)m(\log m)^2$. Then \eqref{eq:use_bern} follows from Lemma \ref{lem:bern}.

\subsection{Proof of Theorem \ref{thm:gfs}}\label{sec:proof_gfs}
Lemma \ref{lem:hX2} follows trivially from the definitions of $\Omega_1$ and $\Omega_2$.
\begin{lemma} \label{lem:hX2}
Suppose $\Omega_{12} = \Omega_1 \bigcap \Omega_2 = \{X\in\bbR^{n_1\times n_2}: \norm{X^{(:,k)}}_0\leq s_1,~\forall k\in [n_2],~\norm{X}_{0,c}\leq s_2\}$. Then
\[\hX_2 = P_{\Omega_2} P_{\Omega_1} \hX_0 = P_{\Omega_{12}} \hX_0.\] 
\end{lemma}

\begin{lemma} \label{lem:max}
Suppose $\Delta_{12} = (\Omega_{12}-\Omega_{12}) \bigcap \calB_{n_1\times n_2}$. Then
\[
\norm{Y}_{\Delta_{12}^\circ} \leq \min\left\{\norm{Y}_\rmF,~\sqrt{2s_1s_2} \max_{j,k}\left|Y^{(j,k)}\right|\right\}.
\]
\end{lemma}
\begin{proof}
Similar to Lemma \ref{lem:max2}, $\norm{Y}_{\Delta_{12}^\circ} \leq \norm{Y}_\rmF$.
Since
\[
\Delta_{12} \subset \{X\in\bbR^{n_1\times n_2}: \norm{X}_0\leq 2s_1s_2, \norm{X}_\rmF\leq 1\} \subset \{X\in\bbR^{n_1\times n_2}: \norm{\vect(X)}_1\leq \sqrt{2s_1s_2}\} \eqqcolon \Delta_{\ell_1},
\]
By H\"{o}lder's inequality,
\begin{align*}
\norm{Y}_{\Delta_{12}^\circ} = \sup\limits_{X\in\Delta_{12}} \left<Y,X\right>\leq \sup\limits_{X\in\Delta_{\ell_1}} \left<Y,X\right> \leq \sup\limits_{X\in\Delta_{\ell_1}} \norm{\vect(X)}_1\norm{\vect(Y)}_\infty = \sqrt{2s_1s_2}\max_{j,k} \left|Y^{(j,k)}\right|.
\end{align*}
\end{proof}

\begin{lemma} \label{lem:order}
Suppose $u_{i}^{(j)}$ ($i=1,2,\cdots,m$, $j=1,2,\cdots,n$) are i.i.d. Gaussian random variables $N(0,1)$. Then
\[
\bbE\left[\max_{j\in[n]} \sqrt{\sum_{i=1}^{m}\sigma_i^2 u_i^{(j)2}}\right] \leq \sqrt{(3\log n + 2) \sum_{i=1}^m \sigma_i^2}.
\]
\end{lemma}

\begin{proof}
Let $d^{(j)} \coloneqq \sqrt{\sum_{i=1}^{m}\sigma_i^2 u_i^{(j)2}}$, and $d \coloneqq \max_{j\in[n]} d^{(j)}$.
By Jensen's inequality,
\[
e^{t\bbE[d^2]} \leq \bbE\left[e^{td^2}\right] \leq \sum_{j=1}^{n}\bbE\left[e^{td^{(j)2}}\right] = n \prod_{i=1}^m\bbE\left[e^{t\sigma_i^2 u_i^{(1)2}}\right] =  n \prod_{i=1}^m \left(1-2t\sigma_i^2\right)^{-\frac{1}{2}},\quad \forall~ 0<t<\frac{1}{2\max_i\sigma_i^2}.
\]
Therefore,
\[
\bbE[d^2] \leq \frac{\log n}{t} - \frac{1}{2t}\sum_{i=1}^{m}\log (1-2t\sigma_i^2),\quad \forall~ 0<t<\frac{1}{2\max_i\sigma_i^2}.
\]
It is easy to verify that $-\frac{1}{2}\log(1-2x) \leq 2x$ for $0<x<\frac{1}{3}$. Choose $t = \frac{1}{3\sum_{i=1}^m \sigma_i^2}$, then $0<t\sigma_i^2<\frac{1}{3}$. Hence
\[
\bbE[d^2] \leq \frac{\log n}{t} +\frac{1}{t}\sum_{i=1}^m 2t\sigma_i^2 = (3\log n + 2) \sum_{i=1}^m \sigma_i^2,
\]
\[
\bbE[d] \leq \sqrt{\bbE[d^2]} \leq \sqrt{(3\log n + 2) \sum_{i=1}^m \sigma_i^2}.
\]
\end{proof}

\textbf{Next, we prove Theorem \ref{thm:gfs}.}
By \eqref{eq:split} and triangle inequality,
\begin{align}
\norm{\hX_0 - UQV^T}_{\Delta_{12}^\circ} \leq &~ \norm{ U \left(\frac{1}{m}\sum_{i=1}^m\ba_i y_i \bb_i^T -Q\right)V^T}_{\Delta_{12}^\circ} + \norm{\tU\left(\frac{1}{m}\sum_{i=1}^m\ta_i y_i \tb_i^T\right)\tV^T}_{\Delta_{12}^\circ}   \nonumber\\
&~ + \norm{U \left(\frac{1}{m}\sum_{i=1}^m\ba_i y_i \tb_i^T \right)\tV^T}_{\Delta_{12}^\circ} + \norm{\tU\left(\frac{1}{m}\sum_{i=1}^m\ta_i y_i \bb_i^T\right)V^T }_{\Delta_{12}^\circ} \nonumber\\
\eqqcolon &~ T_1+T_2+T_3+T_4. \label{eq:split2}
\end{align}
Next, we bound the expectation of the four terms. Similar to \eqref{eq:T1},
\[
\bbE[T_1] \leq \sqrt{\frac{r^2\sigma_{y|a,b}^2+\sigma^2}{m}}.
\]

Suppose $u_i\sim N(0,I_{n_1})$, $v_i\sim N(0,I_{n_2})$, $\{u_i\}_{i=1}^m$, $\{v_i\}_{i=1}^m$, and $\{y_i\}_{i=1}^m$, $\{\ba_i\}_{i=1}^m$, $\{\bb_i\}_{i=1}^m$ are independent. Replacing $\tU\ta_i,\tV\tb_i$ in $T_2$ by $u_i,v_i$, by Lemmas \ref{lem:degen_gauss} and \ref{lem:max},
\begin{align}
\bbE[T_2] \leq \bbE\left[\norm{\frac{1}{m}\sum_{i=1}^m u_i y_i v_i^T}_{\Delta_{12}^\circ}\right]\leq \frac{\sqrt{2s_1s_2}}{m}~ \bbE\left[\max_{j,k} \Big|\sum_{i=1}^m u_i^{(j)}y_i v_i^{(k)}\Big|\right]. \label{eq:T2_first}
\end{align}
Conditioned on $\{y_i,v_i\}_{i=1}^{m}$, the distribution of $\sum_{i=1}^m u_i^{(j)}y_i v_i^{(k)}$ is $N(0,\sum_{i=1}^m y_i^2 v_i^{(k)2})$. By Lemma \ref{lem:order},
\begin{align*}
\bbE\left[\max_{j,k} \Big|\sum_{i=1}^m u_i^{(j)}y_i v_i^{(k)}\Big|~ \middle|~ \{y_i,v_i\}_{i=1}^{m}\right] \leq &~ \max_k \sqrt{(3\log n_1 +2)}\cdot\sqrt{\sum_{i=1}^m y_i^2 v_i^{(k)2}}\\
 \leq &~ 2\sqrt{\log n_1} \max_k \sqrt{\sum_{i=1}^m y_i^2 v_i^{(k)2}}.
\end{align*}
The second line follows from $n_1\geq 8$. Conditioned on $\{y_i\}_{i=1}^{m}$ alone, apply Lemma \ref{lem:order} one more time,
\begin{align*}
\bbE\left[\max_{j,k} \Big|\sum_{i=1}^m u_i^{(j)}y_i v_i^{(k)}\Big|~ \middle|~ \{y_i\}_{i=1}^{m}\right] \leq &~ 2\sqrt{\log n_1} \bbE\left[ \max_k \sqrt{\sum_{i=1}^m y_i^2 v_i^{(k)2}} ~ \middle|~ \{y_i\}_{i=1}^{m}\right] \\
\leq &~ 4\sqrt{\log n_1 \log n_2} \sqrt{\sum_{i=1}^m y_i^2}.
\end{align*}
By \eqref{eq:T2_first},
\begin{align*}
\bbE[T_2] \leq &~ \frac{\sqrt{2s_1s_2}}{m}~ \bbE\left[\max_{j,k} \Big|\sum_{i=1}^m u_i^{(j)}y_i v_i^{(k)}\Big|\right] \\
\leq &~ \frac{4\sqrt{2s_1s_2\log n_1\log n_2}}{m} \bbE\left[ \sqrt{\sum_{i=1}^m y_i^2} \right]\\
\leq &~ 4\sqrt{\frac{2s_1s_2\log n_1\log n_2 \cdot (\sigma_{y|a,b}^2+\tau_0^2)}{m}}.
\end{align*}

The bounds on the expectations of $T_3$ and $T_4$ can be derived similarly. 
\begin{align*}
\bbE[T_3] \leq &~ \bbE\left[\norm{\frac{1}{m}\sum_{i=1}^m U\ba_i y_i v_i^T}_{\Delta_{12}^\circ}\right] \leq \frac{\sqrt{2s_1s_2}}{m}~ \bbE\left[\max_{j,k} \Big|\sum_{i=1}^m (U\ba_i y_i)^{(j)}v_i^{(k)}\Big|\right] \\
\leq &~ \frac{2\sqrt{2s_1s_2\log n_2}}{m}~\bbE\left[\max_{j} \sqrt{\sum_{i=1}^m (U\ba_i y_i)^{(j)2} } \right] \\
\leq &~ \frac{2\sqrt{2s_1s_2\log n_2}}{m}~\bbE\left[\sqrt{\sum_{j=1}^{n_1}\sum_{i=1}^m (U\ba_i y_i)^{(j)2} } \right] \\
\leq &~ \frac{2\sqrt{2s_1s_2\log n_2}}{m}~\sqrt{\bbE\left[\sum_{i=1}^m \norm{\ba_i y_i}_2^2 \right]} \\
\leq &~ 2\sqrt{\frac{2s_1s_2\log n_2 \cdot (r\sigma_{y|a,b}^2+\tau_1^2)}{m}}. \\
\bbE[T_4] \leq &~ 2\sqrt{\frac{2s_1s_2\log n_1 \cdot(r\sigma_{y|a,b}^2+\tau_2^2)}{m}}.
\end{align*}

By Lemmas \ref{lem:proj} and \ref{lem:hX2}, and \eqref{eq:split2}, we have
\begin{align*}
\bbE\left[\norm{\hX_2 - UQV^T}_\rmF\right] \leq &~ 2 \bbE\left[\norm{\hX_0 - UQV^T}_{\Delta_{12}^\circ}\right] \leq 2 \bbE[T_1] + 2 \bbE[T_2] + 2 \bbE[T_3] + 2 \bbE[T_4] \\
\leq &~ 2\sqrt{\frac{r^2\sigma_{y|a,b}^2+\sigma^2}{m}} + 8\sqrt{\frac{2s_1s_2\log n_1\log n_2 \cdot (\sigma_{y|a,b}^2+\tau_0^2)}{m}} \\
&~ + 4\sqrt{\frac{2s_1s_2\log n_2 \cdot (r\sigma_{y|a,b}^2+\tau_1^2)}{m}} + 4\sqrt{\frac{2s_1s_2\log n_1 \cdot(r\sigma_{y|a,b}^2+\tau_2^2)}{m}} \\
= &~ O\left(\sqrt{\frac{s_1s_2\log n_1\log n_2}{m}}\right).
\end{align*}
The last line is due to the fact that $r,\sigma_{y|a,b},\sigma,\tau_0,\tau_1,\tau_2$ are all independent of $n_1$, $n_2$, and $m$.

Since $\hX_3 = P_{\Omega_3} \hX_2$, and $UQV^T\in \Omega_3$, we have
\[
\norm{\hX_3-UQV^T}_\rmF\leq \norm{\hX_3-\hX_2}_\rmF + \norm{\hX_2-UQV^T}_\rmF\leq 2\norm{\hX_2-UQV^T}_\rmF.
\]
Similarly, $\hU\hSigma\hV^T = P_{\Omega_r}\hX_3$, and $UQV^T\in \Omega_r$, hence
\[
\norm{\hU\hSigma\hV^T-UQV^T}_\rmF\leq \norm{\hU\hSigma\hV^T-\hX_3}_\rmF + \norm{\hX_3-UQV^T}_\rmF\leq 2\norm{\hX_3-UQV^T}_\rmF \leq 4\norm{\hX_2-UQV^T}_\rmF.
\]
By Lemma \ref{lem:subspace},
\begin{align*}
\max\left\{\bbE\left[d(U,\hU)\right],~\bbE\left[d(V,\hV)\right] \right\} &\leq \frac{1}{\sigma_r}\bbE\left[\norm{UQV^T-\hU\hSigma\hV^T}_\rmF\right] \leq \frac{4}{\sigma_r} \bbE\left[\norm{\hX_2 - UQV^T}_\rmF\right] \\
& = O\left(\sqrt{\frac{s_1s_2\log n_1\log n_2}{m}}\right).
\end{align*}

\section{Mildness of the Light-tailed Response Condition} \label{sec:mild}

In this section, we demonstrate that this condition holds under reasonably mild assumptions on $f(\cdot,\cdot)$ and $y-\mu_{y|(a,b)}$. To this end, we review a known fact: a probability distribution is light-tailed if its moment generating function is finite at some point. This is made more precise in Proposition \ref{pro:light}, which follows trivially from Chernoff bound.
\begin{proposition}\label{pro:light}
Let $M_y(t) = \bbE\left[e^{ty}\right]$ denote the moment generating function of a random variable $y$. Then $y$ is a light-tailed random variable, if
\begin{itemize}
\item there exist $t_1>0$ and $t_2<0$ such that $M_y(t_1)<\infty$ and $M_y(t_2)<\infty$.
\item $y\geq 0$ almost surely, and there exists $t_1>0$ such that $M_y(t_1)<\infty$.
\item $y\leq 0$ almost surely, and there exists $t_2<0$ such that $M_y(t_2)<\infty$.
\end{itemize}
\end{proposition}
In the context of this paper, we have the following corollary:
\begin{corollary}\label{cor:light}
Suppose $f(\ba,\bb)$ satisfies $|f(\ba,\bb)|\leq \max\left\{C_1,~C_2\left(\norm{\ba}_2^2+\norm{\bb}_2^2\right)\right\}$ for some $C_1,C_2>0$, and $y-\mu_{y|(a,b)} = y-f(\ba,\bb)$ is a light-tailed random variable. Then $y$ is a light-tailed random variable.
\end{corollary}

\begin{proof}
Since $\bbP\left[\left|y\right|\geq t\right] \leq \bbP\left[\left|\mu_{y|(a,b)}\right|\geq t/2\right]+\bbP\left[\left|y-\mu_{y|(a,b)}\right|\geq t/2\right]$, and $y-\mu_{y|(a,b)}$ is light-tailed, it is sufficient to show that $\mu_{y|(a,b)}$ is light-tailed. The moment generating function of $\mu_{y|(a,b)}$ is
\begin{align*}
M_\mu(t) = &~ \bbE[e^{tf(\ba,\bb)}] \leq  \bbE[e^{|t|\cdot|f(\ba,\bb)|}] \leq e^{C_1|t|}\bbE\left[e^{C_2|t|\left(\norm{\ba}_2^2+\norm{\bb}_2^2\right)}\right] \\
= &~ \frac{e^{C_1|t|}}{(2\pi)^r} \int\limits_{\bb}\int\limits_{\ba} e^{\left(C_2|t|-\frac{1}{2}\right)\left(\norm{\ba}_2^2+\norm{\bb}_2^2\right)}~ \rmd \ba ~\rmd \bb,
\end{align*}
which is finite for $|t|<\frac{1}{2C_2}$. By Proposition \ref{pro:light}, $\mu_{y|(a,b)}$ is light-tailed. Thus the proof is complete.
\end{proof}

\section{Estimation of Rank and Sparsity}
Throughout the paper, we assume that the rank $r$ and sparsity levels $s_1,s_2$ are known. In practice, these parameters often need to be estimated from data, or selected by user. In this section, we give a partial solution to parameter estimation.

If the sample complexity satisfies $m=\Omega(n_1n_2)$, then $r$, $s_1$, $s_2$ can be estimated from $\hX_0$ as follows. Let $(J,K)$ and $(J,K)^c$ denote the support of $X=UQV^T$ (the set of indices where $X$ is nonzero) and its complement. Let $\sigma_i(\cdot)$ denote the $i$-th singular value of a matrix. Suppose for some $\eta>0$,
\[
\min_{(j,k)\in(J,K)}|X^{(j,k)}|\geq \eta, \quad \sigma_r(X) = \sigma_r(Q) \geq \eta.
\]
By Theorem \ref{thm:lin}, we can achieve $\|\hX_0-X\|_\rmF \leq \frac{1}{3}\eta$ with $m=\Omega(n_1n_2)$ samples. Then
\begin{align*}
\min_{(j,k)\in(J,K)}|\hX_0^{(j,k)}| \geq \frac{2}{3}\eta, \quad \max_{(j,k)\in(J,K)^c}|\hX_0^{(j,k)}| \leq \frac{1}{3}\eta,\quad
\sigma_r(\hX_0) \geq \frac{2}{3}\eta, \quad \sigma_{r+1}(\hX_0) \leq \frac{1}{3}\eta.
\end{align*}
Therefore, an entry is nonzero in $X$ if and only if the absolute value of the corresponding entry in $\hX_0$ is greater than $\frac{1}{2}\eta$. We can determine $s_1$ and $s_2$ by counting the number of such entries. Similarly, the rank $r$ of matrix $X$ can be determined by counting the number of singular values of $\hX_0$ greater than $\frac{1}{2}\eta$. In practice, such a threshold $\eta$ is generally unavailable. However, by gathering a sufficiently large number of samples, the entries and singular values of $\hX_0$ will vanish if the corresponding entries and singular values in $X$ are zero, thus revealing the true sparsity and rank.

In practice, one can select parameters $r$, $s_1$, and $s_2$ based on cross validation.


\section{Pathological Cases for Supervised Dimensionality Reduction}

To estimate the embedding matrices, the matrix $Q\in\bbR^{r\times r}$, which depends on $f(\cdot,\cdot)$, needs to be non-singular. This means, as revealed by our analysis, our algorithms fail in the pathological case where $f(\ba,\bb)$ is even in $\ba$ or $\bb$. However, previous supervised DR methods suffer from similar pathological cases. For example, principal Hessian direction (pHd) \cite{Li1992} fails when $f(\cdot,\cdot)$ is odd in both variables. 
We compare our approach with pHd for two link functions: 1) $f(\ba,\bb) = \ba^T\bb = \sum_{j=1}^r \ba^{(j)}\bb^{(j)}$, which is odd in $\ba,\bb$, and 2) $f(\ba,\bb) = \sum_{j=1}^r \ba^{(j)2}\bb^{(j)2}$, which is even in $\ba,\bb$. The results in Figure \ref{fig:compare_pHd} show that: 1) For the odd function, our approach succeeds, but pHd fails; 2) For the even function, our approach fails, but pHd succeeds.

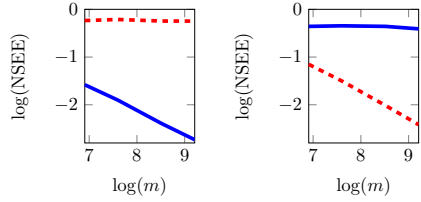
\begin{figure}[htbp]
\centering
%
%
\begin{tikzpicture}[scale=0.7]

\begin{axis}[%
width=0.822in,
height=1in,
at={(0in,0in)},
scale only axis,
xmin=6.908,
xmax=9.21,
xlabel={$\log(m)$},
ymin=-2.8,
ymax=0,
ylabel={$\log(\mathrm{NSEE})$},
axis background/.style={fill=white}
]
\addplot [color=blue,solid,line width=2.0pt,forget plot]
  table[row sep=crcr]{%
6.90775527898214	-1.58020068553952\\
7.60090245954208	-1.90241849989738\\
8.51719319141624	-2.39339524848198\\
9.21034037197618	-2.73179799885045\\
};
\addplot [color=red,dashed,line width=2.0pt,forget plot]
  table[row sep=crcr]{%
6.90775527898214	-0.235154733757411\\
7.60090245954208	-0.212884630672413\\
8.51719319141624	-0.245352264966206\\
9.21034037197618	-0.246449811460298\\
};
\end{axis}
\end{tikzpicture}%
~~~~
%
%
\begin{tikzpicture}[scale=0.7]

\begin{axis}[%
width=0.822in,
height=1in,
at={(0in,0in)},
scale only axis,
xmin=6.908,
xmax=9.21,
xlabel={$\log(m)$},
ymin=-2.8,
ymax=0,
ylabel={$\log(\mathrm{NSEE})$},
axis background/.style={fill=white}
]
\addplot [color=blue,solid,line width=2.0pt,forget plot]
  table[row sep=crcr]{%
6.90775527898214	-0.357542349616585\\
7.60090245954208	-0.345415705555318\\
8.51719319141624	-0.35742561210398\\
9.21034037197618	-0.409221455986557\\
};
\addplot [color=red,dashed,line width=2.0pt,forget plot]
  table[row sep=crcr]{%
6.90775527898214	-1.1528123149822\\
7.60090245954208	-1.50108420216452\\
8.51719319141624	-2.0156654517911\\
9.21034037197618	-2.41665463667388\\
};
\end{axis}
\end{tikzpicture}%
\caption{Log-log plots of our approach (blue solid lines) versus pHd (red dashed lines). The left plot is for an odd function, and the right plot is for an even function.
}
\label{fig:compare_pHd}
\end{figure}

In practice, such pathological cases, for which an algorithm completely fails to recover any useful direction in the dimensionality reduction subspace, are very rare. In most applications, depending on the underlying link functions, different DR methods recover useful embedding matrices to varying degrees. In this sense, our algorithm is a complement to previous supervised DR methods.

\end{document}